%% file: main.tex
\documentclass[nohyperref]{article}

\usepackage[accepted]{icml2022}

\input{packages}

\input{notations}

\icmltitlerunning{Bayesian Optimization under Stochastic Delayed Feedback}

\begin{document}
	\twocolumn[
	\icmltitle{Bayesian Optimization under Stochastic Delayed Feedback} 

	\icmlsetsymbol{equal}{*}
	
	\begin{icmlauthorlist}
		\icmlauthor{Arun Verma}{equal,nus}
		\icmlauthor{Zhongxiang Dai}{equal,nus}
		\icmlauthor{Bryan Kian Hsiang Low}{nus}
	\end{icmlauthorlist}
	
	\icmlaffiliation{nus}{Department of Computer Science, National University of Singapore, Republic of Singapore}
	
	\icmlcorrespondingauthor{Arun Verma}{arunverma100@gmail.com}
	\icmlcorrespondingauthor{Zhongxiang Dai}{daiz9109@gmail.com}
	
	\icmlkeywords{Bayesian Optimization, Stochastic Delayed Feedback, Batch Bayesian Optimization, Contextual Gaussian Process Bandits}
	
	\vskip 0.3in
	]
	
	\printAffiliationsAndNotice{\icmlEqualContribution} 
	
	\begin{abstract}
		Bayesian optimization (BO) is a widely-used sequential method for zeroth-order optimization of complex and expensive-to-compute black-box functions. The existing BO methods assume that the function evaluation (feedback) is available to the learner immediately or after a fixed delay. Such assumptions may not be practical in many real-life problems like online recommendations, clinical trials, and hyperparameter tuning where feedback is available after a random delay. To benefit from the experimental parallelization in these problems, the learner needs to start new function evaluations without waiting for delayed feedback. In this paper, we consider the BO under stochastic delayed feedback problem. We propose algorithms with sub-linear regret guarantees that efficiently address the dilemma of selecting new function queries while waiting for randomly delayed feedback. Building on our results, we also make novel contributions to batch BO and contextual Gaussian process bandits. Experiments on synthetic and real-life datasets verify the  performance of our algorithms.	
	\end{abstract}

	\section{Introduction}
	\label{sec:introduction}
	\input{latex/introduction}

	\section{Problem Setting}		
	\label{sec:problem_setting}
	\input{latex/problem_setting}

	\section{BO under Stochastic Delayed Feedback}
	\label{sec:bo}
	\input{latex/bo}

	\section{Improved Algorithms for Batch BO}
	\label{sec:batch_bo}
	\input{latex/batch_bo}

	\section{Contextual Gaussian Process Bandits under Stochastic Delayed Feedback}
	\label{sec:contx_bandits}
	\input{latex/contx_bandits}

	\section{Experiments}
	\label{sec:experiments}
	\input{latex/experiment}

	\section{Related Work}
	\label{sec:related_work}
	\input{latex/related_work}

	\section{Conclusion}
	\label{sec:conclusion}
	\input{latex/conclusion}

	\section*{Acknowledgements}
    This research/project is supported by A*STAR under its RIE$2020$ Advanced Manufacturing and Engineering (AME) Industry Alignment Fund – Pre Positioning (IAF-PP) (Award A$19$E$4$a$0101$) and by the Singapore Ministry of Education Academic Research Fund Tier $1$.

	\bibliography{ref}
	\bibliographystyle{icml2022}

	\newpage
	\onecolumn
	\centerline{\Large \bf Appendix} 
	\vskip .15in
	\hrule height1pt 
	
	\appendix
	\label{sec:appendix}
	\input{latex/appendix}

\end{document}

%% file: packages.tex
\usepackage[utf8]{inputenc}   
\usepackage[T1]{fontenc}    	
\usepackage{url}              		 
\usepackage{booktabs}       	
\usepackage{amsfonts}       	
\usepackage{nicefrac}       	  
\usepackage{microtype}      	
\usepackage[english]{babel}

\usepackage{graphicx}
\usepackage{color}
\usepackage{rotating}
\usepackage{tabularx}
\usepackage{pdflscape}
\usepackage{amsmath,amsthm,amssymb}
\usepackage{algorithm, algorithmic}
\usepackage{bbm,dsfont}
\usepackage{subfig}
\usepackage{wrapfig}
\usepackage{cancel}
\usepackage{enumerate, cases}
\usepackage{thmtools,thm-restate}
\usepackage{mathtools}

\usepackage[none]{hyphenat}
\usepackage{multirow}
\usepackage{makecell}
\usepackage{setspace}
\usepackage{pifont}
\usepackage{array}
\usepackage{enumitem}

\allowdisplaybreaks

\usepackage{hyperref}		 
\hypersetup{
	colorlinks	= true,		 
	urlcolor     = blue,	 
	linkcolor	 = purple, 	 
	citecolor    = violet    
}
\usepackage[capitalize]{cleveref}


%% file: notations.tex
\newcommand{\Regret}{\kR}

\newcommand{\Prob}[1]{\bP\left\{#1\right\}}

\newcommand{\R}{\bR}

\renewcommand{\phi}{\varphi}
\renewcommand{\epsilon}{\varepsilon}

\newcommand{\norm}[1]{\left\Vert #1 \right\Vert}

\newcommand{\al}[1]{ \begin{align} #1  \end{align}}
\newcommand{\eq}[1]{ \begin{equation} #1  \end{equation}}
\newcommand{\als}[1]{ \begin{align*} #1  \end{align*}}
\newcommand{\eqs}[1]{ \begin{equation*} #1  \end{equation*}}

\newcommand{\Lp}{\left(}
\newcommand{\Rp}{\right)}
\newcommand{\Lb}{\left[}
\newcommand{\Rb}{\right]}

\newcommand{\el}{\end{flushleft}}
\newcommand{\bl}{\begin{flushleft}}

\newcommand{\argmax}{\arg\!\max}

\newcommand{\bE}{\mathbb{E}}

\newcommand{\bN}{\mathbb{N}}

\newcommand{\bP}{\mathbb{P}}

\newcommand{\bR}{\mathbb{R}}

\newcommand{\cB}{\mathcal{B}}

\newcommand{\cD}{\mathcal{D}}

\newcommand{\cQ}{\mathcal{Q}}

\newcommand{\cZ}{\mathcal{Z}}

\newcommand{\kR}{\mathfrak{R}}

\theoremstyle{plain}

\newtheorem{lem}{Lemma}

\newtheorem{defi}{Definition}

%% file: latex/introduction.tex

Bayesian optimization (BO) \citep{IEEE15_shahriari2015taking,Arxiv18_frazier2018tutorial,Book_garnett2021bayesian} is a popular and widely-used sequential method for zeroth-order optimization of unknown black-box functions that are complex and expensive to compute. The existing BO methods assume that the function evaluation (feedback) is available to the learner immediately or after a fixed delay. 
However, these assumptions are impractical in many real-life problems where feedback is available after a \emph{random delay}. To take advantage of the experimental parallelization in these problems, the learner needs to start new function evaluations without waiting for the randomly delayed feedback. 
We refer to this new BO problem as {\it `Bayesian Optimization under Stochastic Delayed Feedback'} (BO-SDF). In this paper, we propose algorithms that efficiently address the problem of selecting new function queries while waiting for randomly delayed feedback. Specifically, we answer the following question:
\newline
{\bf \emph{How to start a new function query when the observations of the past function queries are randomly delayed?}}

Many real-life applications can be cast as BO-SDF problems. 
For example, when performing clinical trials to discover new medicines, we need to optimize the amount of different raw materials in order to find their most effective composition using a small number of clinical trials \citep{Book_chow2006adaptive}, which is a natural application for BO due to its sample efficiency.
However, the evaluation of the effectiveness of a medicine needs to take into account the side effects which are usually not observed immediately but are instead revealed \emph{after a random period} due to the varying physiology of different patients.
Therefore, a natural challenge here is how to use BO to choose the next composition for testing, while accounting for the stochastically delayed observations from some previous clinical trials.
A similar challenge also arises when we aim to find the most effective dose of a new medicine through clinical trials \citep{PS21_takahashi2021bayesian}.

Another motivating example arises from online product recommendation \citep{KDD14_chapelle2014modeling,ADKDD17_diemert2017attribution}. Most online platforms make recommendations in milliseconds, but the user's response (i.e., watching a movie or buying a product) generally happens after a random time period which may range from hours to days or even weeks. Furthermore, the issue of randomly delayed feedback also plagues other common applications of BO such as hyperparameter tuning of machine learning (ML) models \citep{NIPS12_snoek2012practical} (e.g., the training time of a neural network depends on the number of layers, network width, among others) and material discovery \citep{MD16_ueno2016combo}.

The closest BO setting to our BO-SDF is batch BO \citep{desautels2014parallelizing,daxberger2017distributed,chowdhury2019batch} which also allows multiple queries to be performed on the black-box function in parallel. 
Both problems require choosing an input query to evaluate the black-box function while the observations of some previously selected inputs are not available (i.e., delayed).
However, there are important differences because the delays are fixed in a certain way in batch BO but can be random in BO-SDF.
Interestingly, batch BO can be considered as a special case of our BO-SDF problem and hence, our algorithms proposed for BO-SDF can be applied to batch BO while achieving important improvements (\cref{sec:batch_bo}).

Following the practice of BO, in order to choose the input queries to quickly approach the global optima, we employ a \emph{Gaussian process} (GP) \citep{Book_rasmussen2003gaussian} to model the black-box function, which builds a posterior belief of the function using the history of function queries and their feedback.
However, when constructing this posterior belief in BO-SDF problems, we are faced with the important question of {\it how to handle the randomly delayed feedback whose observations are not available}.
To this end, we replace the unavailable delayed feedback by the minimum function value,\footnote{In many problems, the minimum value of a function is known, e.g., a user's minimum response in online recommendation systems is `no click' (i.e., $0$) and the minimum accuracy for hyperparameter tuning of machine learning models is `$0$'.} 
which we refer to as \emph{censored feedback.}
The use of censored feedback, interestingly, improves the exploration in BO-SDF problems (\cref{subsec:gp:and:censored:feedback}) and 
hence leads to better theoretical and empirical performances. With the posterior belief built using censored feedback, we propose algorithms using upper confidence bound (UCB) and Thompson sampling (TS), both of which enjoy sub-linear regret guarantees.
Specifically, our contributions are as follows:
\vspace{-3mm}
\begin{itemize}
	\setlength\itemsep{-0.07em}
	\item We introduce and formalize the notion of \emph{censored feedback} in the BO-SDF problem, and propose UCB- and TS-based algorithms with sub-linear regret guarantees (\cref{sec:bo}). 
	
	\item Since batch BO is a special case of BO-SDF, we apply our proposed algorithms (\cref{sec:bo}) to batch BO and show that our algorithms enjoy tighter regret upper bounds than classic batch BO algorithms (\cref{sec:batch_bo}). This gain is mainly attributed to the censored feedback, which leads to a better exploration policy.
	
	\item We extend our algorithms for the BO-SDF problem to contextual Gaussian process bandits \citep{NIPS11_krause2011contextual} with stochastic delayed feedback in \cref{sec:contx_bandits}. This new contribution is itself of independent interest.

	\item Our experimental results in \cref{sec:experiments} validate the different performance aspects of our proposed algorithms on synthetic and real-life datasets.
\end{itemize}

%% file: latex/problem_setting.tex

This paper considers the Bayesian optimization (BO) problem where the function evaluations (feedback) are available to the learner after a random delay.
Let $\cQ \subset \R^n$ be a finite domain of all function queries where $n \ge 1$.\footnote{We assume $\cQ$ to be finite, but it is straightforward to extend our theoretical results to a compact domain using a suitable discretization scheme \citep[Lemma 2]{Arxiv21_li2021gaussian}.}
The learner selects a new function query as soon as enough resources are available (or when experiment parallelization is possible). 
We denote the $t$-th query by $x_t \in \cQ$.
After the learner selects the query $x_t$ to evaluate the unknown black-box function $f$, the environment generates a noisy function evaluation, denoted by feedback $y_t = f(x_t) + \epsilon_t$.  
We assume that $\epsilon_t$ is an $R$-sub-Gaussian noise. The learner observes $y_t$ only after a stochastic delay $d_t$, which is generated from an unknown distribution $\cD$.

We refer to this new BO problem as \emph{`Bayesian Optimization under Stochastic Delayed Feedback'} (BO-SDF).
The unknown function $f$, query space $\cQ$, and unknown delay distribution $\cD$ identify an instance of the BO-SDF problem. The optimal query ($x^\star$) has the maximum function value (global maximizer), i.e., $x^\star \in \argmax_{x \in \cQ} f(x)$. After selecting a query $x_t$, the learner incurs a penalty (or \emph{instantaneous regret}) $p_t = f(x^\star) - f(x_t)$.

Since the optimal function query is unknown, we sequentially estimate this using the available information of the selected queries and observed feedback. 
Our goal is to learn a sequential policy for selecting queries that finds the optimal query (or global maximizer) as quickly as possible. There are two common performance measures for evaluating a sequential policy. 
The first performance measure is \emph{simple regret}. Let $x_t$ be the $t$-th function query selected by the policy. Then, after observing $T$ function evaluations, the simple regret is 
$
    r_T = f(x^\star) - \max_{t \in \{1,\ldots, T\}}f(x_t).
$ 
Any good policy must have no regret, i.e., $\lim_{T \to \infty}{r_T} = 0$. 
The second performance measure of the policy is \emph{cumulative regret} which is the sum of total penalties incurred by the learner.
After observing $T$ function evaluations, the cumulative regret of a policy is given by 
$
    \Regret_T = \sum_{t=1}^T \left(f(x^\star) - f(x_t)\right).
$
Any good policy should have sub-linear regret, i.e., $\lim_{T \to \infty}{\Regret_T}/T = 0.$
Even though simple regret and cumulative regret are different performance measures, having a policy with no regret or sub-linear regret implies that the policy will eventually converge to the optimal query (or global optimizer).

%% file: latex/bo.tex

A function estimator is the main component of any BO problem for achieving good performance. We use a Gaussian process (GP) as a surrogate for the posterior belief of the unknown function \citep{Book_rasmussen2003gaussian,ICML10_srinival2010gaussian}. To deal with the delayed feedback, we  will introduce the notion of \emph{censored feedback} in GPs, in which the delayed feedback is replaced by the minimum function value. Finally, we will propose UCB- and TS-based algorithms with sub-linear regret guarantees.

\subsection{Estimating Function using Gaussian Processes}
\label{subsec:gp:and:censored:feedback}
We assume that the unknown function $f$ belongs to the \emph{reproducing kernel Hilbert space} (RKHS) associated with a kernel $k$, i.e., $\norm{f}_{k}\leq \cB_f$ where $\cB_f>0$. By following standard practice in the literature, we assume without loss of generality that $k(x,x')\leq 1$ for all $x,x'\in\mathcal{Q}$. We also assume that feedback is bounded, i.e., $|y_t| \leq \cB_y$ for all $t \ge 1$. This assumption is not restrictive since any bounded function can be re-scaled to satisfy this boundedness requirement. For example, the validation accuracy of machine learning models in our hyperparameter tuning experiment in \cref{sec:experiments} is bounded in $[0,1]$.

\paragraph{Censored Feedback.} 
The learner has to efficiently exploit available information about queries and feedback to achieve a better performance. The estimated GP posterior belief has two parts, i.e., mean function and covariance function. The posterior covariance function only needs the function queries for building tighter confidence intervals, whereas the posterior mean function needs both queries and their feedback. 
Incorporating the queries with delayed feedback in updating the posterior mean function leads to a better estimate, but the question is \emph{how to do it}. One possible solution is to replace the delayed feedback with some value, but then the question is \emph{what should that value be}.

To pick the suitable value for replacing delayed feedback, we motivate ourselves from real-life applications like the online movies recommendations problem. 
When an online platform recommends a movie, it expects the following responses from the users -- not watching the movie, watching some part of the movie, and watching the entire movie. 
The platform can map these responses to $[0, 1]$ where $0$ is assigned for `not watching the movie', $1$ for `watching the entire movie', and an appropriate value in $(0,1)$ for `watching some part of the movie.'
There are two reasons for the delay in the user's response: The user does not like the recommended movie or does not have time to watch the movie. 
Before the user starts watching the movie, the platform can replace the user's delayed response with `not watching the movie' (i.e., `$0$') and update it later when more information is available. Therefore, we can replace the delayed feedback with the minimum function value.\footnote{
It is possible to replace the delayed feedback with other values. When it is replaced by the current GP posterior mean, it has the same effect as the \emph{hallucination} in batch BO (more details in \cref{sec:batch_bo}). Another possible value for replacing the delayed feedback is the $k$-nearest posterior means where $k>0$.}
We refer to this replaced feedback as \emph{censored feedback}. 
The idea of censored feedback is also used in other problems such as censoring delayed binary feedback \citep{ICML20_vernade2020linear} and censoring losses \citep{NeurIPS19_verma2019censored, Arxiv21_verma2021censored} or rewards \citep{INFOCOM20_verma2020stochastic} in online resource allocation problems.

When the minimum function value replaces the delayed feedback, the posterior mean around that query becomes small, consequently ensuring that the learner will not select the queries around the queries with delayed feedback before exploring the other parts of the query's domain. Therefore, the censored feedback leads to better exploration than replacing delayed feedback with a larger value (e.g., current posterior mean). 
However, the queries with delayed feedback need to be stored so that the posterior mean can be updated when feedback is revealed. 
Due to several reasons (e.g., limited storage, making resources available to new queries), the learner cannot keep all queries with delayed feedback and replace the old queries with the latest ones.\footnote{Our motivation for removing older queries comes from online recommendation problems where the chance of observing feedback diminishes with time. The learner can also use a more complicated scheme for query removal.} It can degrade the algorithm's performance, as discussed in \cref{ssec:regret_analysis} and demonstrated by experiments in \cref{sec:experiments}.

Let an \emph{incomplete query} represent the function query with unobserved feedback and $m$ be the number of incomplete queries the learner can store. 
We assume that the random delay for a query is the number of new queries that need to be started (i.e., number of iterations elapsed) before observing its feedback. Therefore, the delay $d_t$ only takes values of non-negative integers.\footnote{We have introduced the delay based on \emph{iterations} here to bring out the main ideas of censored feedback. Any suitable discrete distribution (e.g., Poisson) can model such delays. We have also introduced the censored feedback with \emph{time-based} delay, which is more practical and can be modeled using a suitable continuous distribution (see \cref{assec:censoredTime} for more details).} 
The censored feedback of the $s$-th query before selecting the $t$-th query is denoted by $\tilde{y}_{s,t} = y_s\mathbbm{1}\{d_s \leq \min(m, t - s)\}$. That is, the learner censors the incomplete queries by assigning $0$ to them.\footnote{We can make the minimum value of any bounded function $0$ by adding to it the negative of the minimum function value. It does not change the optimizer since the optimizer of a function is invariant to the constant function shift. By doing this, we can represent censored feedback by an indicator function, making censored feedback easier to handle in theoretical analysis.}
Now, the GP posterior mean and covariance functions can be expressed using available function queries and their censored feedback, as follows:
\al{
	\label{eq:gp_posterior}
	\begin{split}
		\mu_{t-1}(x) &= \mathbf{k}_{t-1}(x)^\top\mathbf{K}_{t,\lambda}^{-1}\ \tilde{\mathbf{y}}_{t-1}\ ,\\
		\sigma_{t-1}^2(x,x') &= k(x,x')-\mathbf{k}_{t-1}(x)^\top\mathbf{K}_{t,\lambda}^{-1}\ \mathbf{k}_{t-1}(x')
	\end{split}
}
where $\mathbf{K}_{t,\lambda} = \mathbf{K}_{t-1}+\lambda\mathbf{I}$, $\mathbf{k}_{t-1}(x)= (k(x, x_{t'}))^{\top}_{t'\in[t-1]}$, $\tilde{\mathbf{y}}_{t-1}= (\tilde{y}_{s,t})^{\top}_{s\in[t-1]}$, $\mathbf{K}_{t-1}= (k(x_{t'}, x_{t''}))_{t',t''\in[t-1]}$, and  
$\lambda$ is a regularization parameter that ensures $\mathbf{K}_{t,\lambda}$ is a positive definite matrix. We denote $\sigma_{t-1}^2(x)=\sigma_{t-1}^2(x,x)$.

\subsection{Algorithms for the BO-SDF problem}
After having a posterior belief of the unknown function, the learner can decide which query needs to be selected for the subsequent function evaluation. Since the feedback is only observed for the selected query, the learner needs to deal with the exploration-exploitation trade-off. We use UCB- and TS-based algorithms for BO-SDF problems that handle the exploration-exploitation trade-off efficiently. 

{\bf UCB-based Algorithm for BO-SDF problems.} 
UCB is an effective and widely-used technique for dealing with the exploration-exploitation trade-off in various sequential decision-making problems \citep{ML02_auer2002finite,COLT11_garivier2011kl}. We propose a UCB-based algorithm named \ref{alg:GP-UCB-SDF} for the BO-SDF problems. It works as follows: When the learner is ready to start the $t$-th function query, he firstly updates the GP posterior mean and standard deviation defined in \cref{eq:gp_posterior} using available censored noisy feedback. Then, he selects the input $x_t$ for the next function evaluation by maximizing the following UCB value:
\eq{
	\label{eq:ucb_value}
	\textstyle x_t = \argmax_{x\in\mathcal{Q}} \left(\mu_{t-1} + \nu_t\sigma_{t-1}(x)\right)
}
where $\nu_t = \cB_y\sum^{t-1}_{s=t-m}\sigma_{t-1}(x_s) + \beta_t$, $\beta_t = \cB_f + (R + \cB_y)\sqrt{2(\gamma_{t-1}+1+\log(2/\delta))}$, $\delta\in(0,1)$, and $\gamma_t$ is the maximum information gain about the function $f$ from any set of $t$ function queries \citep{ICML10_srinival2010gaussian}.

\begin{algorithm}[!ht]
	\renewcommand{\thealgorithm}{GP-UCB-SDF}
	\floatname{algorithm}{}
	\caption{UCB-based Algorithm for BO-SDF}
	\label{alg:GP-UCB-SDF}
	\begin{algorithmic}[1]
		\STATE \textbf{Input:} $\lambda > 0$, $m \in \bN$
		\FOR{$t=1,2,3,...$}
			\STATE Update GP posterior defined in \cref{eq:gp_posterior} using available censored noisy feedback
			\STATE Select input $x_t$ using \cref{eq:ucb_value} and compute $f$ at $x_t$
			\STATE Keep observing the noisy feedback until the starting of next function evaluation
		\ENDFOR
	\end{algorithmic}
\end{algorithm}

After selecting the input $x_t$, the learner queries the unknown function $f(x_t)$ and the environment generates a noisy feedback $y_t$ with an associated delay $d_t \sim \cD$. The learner observes $y_t$ only after a delay of $d_t$ iff $d_t \le m$. Before starting a new function evaluation, the learner keeps observing the noisy feedback of past queries. The same process is repeated for the subsequent function evaluations.

{\bf TS-based Algorithm for BO-SDF problems.} 
In contrast to the deterministic UCB, TS \citep{JSTOR33_thompson1933likelihood,COLT12_agrawal2012analysis,AISTATS13_agrawal2013further,ALT12_kaufmann2012thompson} is based on Bayesian updates that select inputs according to their probability of being the best input. Many works have shown that TS is empirically superior to UCB-based algorithms \citep{NIPS11_chapelle2011empirical,chowdhury2019batch}. We name our TS-based algorithm \ref{alg:GP-TS-SDF} which works similarly to \ref{alg:GP-UCB-SDF}, except that the $t$-th function query is selected as follows:
\eq{
	\label{eq:ts_value}
	\textstyle x_t = \argmax_{x \in \cQ}f_t(x)
}
where $f_t \sim \mathcal{GP}(\mu_{t-1}(\cdot),\nu_t^2\sigma_{t-1}^2(\cdot))$.

\begin{algorithm}[!ht]
	\renewcommand{\thealgorithm}{GP-TS-SDF}
	\floatname{algorithm}{}
	\caption{TS-based Algorithm for BO-SDF}
	\label{alg:GP-TS-SDF}
	\begin{algorithmic}[1]
	    \STATE \textbf{Input:} $\lambda > 0$, $m \in \bN$
		\FOR{$t=1,2,3,...$}
			\STATE Update GP posterior defined in \cref{eq:gp_posterior} using available censored noisy feedback 
			\STATE Select input $x_t$ using \cref{eq:ts_value} and compute $f$ at $x_t$
			\STATE Keep observing the noisy feedback until the starting of next function evaluation%
		\ENDFOR
	\end{algorithmic}
\end{algorithm}

\subsection{Regret Analysis}
\label{ssec:regret_analysis}
Let $\rho_m=\mathbb{P}(d_s\leq m)$ for $s\geq1$ be the probability of observing delayed feedback within the next $m$ iterations.
The following result gives the confidence bounds of GP posterior mean function and is important to our subsequent regret analysis:

\begin{restatable}[Confidence Ellipsoid]{thm}{confBound}
	\label{thm:confBound} Let $x\in\cQ$, $\lambda>1$, and $t\geq 1$. Then, with probability at least $1-\delta$, 
	\eqs{
	    \left|\mu_{t-1}(x) - \rho_m f(x)\right|\leq \nu_t \sigma_{t-1}(x).
	}
\end{restatable}
The detailed proofs of \cref{thm:confBound} and other results below are given in Appendix~\ref{app:sec:proof}. Now, we state the regret upper bounds of our proposed \ref{alg:GP-UCB-SDF} and \ref{alg:GP-TS-SDF}:
\begin{restatable}[\ref{alg:GP-UCB-SDF}]{thm}{ucbRegret}
    \label{thm:ucbRegret}
    Let $C_1=\sqrt{{2}/{\log(1+\lambda^{-1})}}$. Then, with probability at least $1-\delta$,
    \eqs{
        \Regret_T \leq \frac{2}{\rho_m}\left( C_1\beta_T\sqrt{T\gamma_T} + m\cB_y C_1^2 \gamma_T \right).
    }
\end{restatable}
As expected, the regret inversely depends on $\rho_m$, i.e., the probability of observing delayed feedback within the next $m$ iterations.
If we ignore the logarithmic factors and constants in \cref{thm:ucbRegret}, then the regret of \ref{alg:GP-UCB-SDF} is upper bounded by $\tilde{O}\Lp {\rho^{-1}_m} \Lp \gamma_T \sqrt{T} + m\gamma_T \Rp \Rp$.

\begin{restatable}[\ref{alg:GP-TS-SDF}]{thm}{tsRegret}
    \label{thm:tsRegret}
    With probability at least $1-\delta$,
    \eqs{
        \Regret_T = \tilde{O}\Lp \frac{1}{\rho_m}\Lp \sqrt{T\gamma_T} (\sqrt{\gamma_T} + 1) + m(\gamma_T + \sqrt{T}) \Rp \Rp.
    }
\end{restatable}
The regret bounds in both Theorems~\ref{thm:ucbRegret} and~\ref{thm:tsRegret} are \emph{sub-linear} for the commonly used squared exponential (SE) kernel, for which $\gamma_T=O(\log^{n+1}(T))$~\cite{ICML10_srinival2010gaussian}. 

By following the common practice in the BO  literature \citep{contal2013parallel,kandasamy2018parallelised}, we can upper-bound the simple regret by the average cumulative regret, i.e., $r_T \le \Regret_T/T$.

{\bf Discussion on the trade-off regarding $m$.}
A large value of $m$ ensures that fewer queries are discarded and favors the regret bound (i.e., via the term $1/\rho_m$) by making $\rho_m$ large.
However, a large $m$ also means that our algorithm allows larger delays, which consequently forces us to enlarge the width of the confidence interval (\cref{thm:confBound}) to account for them. 
It makes the regret bound worse, which is also reflected by the linear dependence on $m$ in the second term of the regret upper bounds (Theorems~\ref{thm:ucbRegret} and~\ref{thm:tsRegret}). 
As $m$ is directly related to the storage requirement of input queries whose feedback are yet to be observed, a large value of $m$ is constrained by the resources available to the learner.

{\bf Input-dependent random delays.}
In some applications, the random delay can depend on the function query: for example, a larger number of hidden layers increases the training time of a neural network. 
Our regret bounds still hold with input-dependent random delays by redefining $\rho_m$ for $T$ queries as $\rho_m = \min_{t \in [T]} \Prob{d_t \le m}$. 

{\bf Limitations of censored feedback.}
Censoring, i.e., replacing values of incomplete queries with the minimum function values, leads to an aggressive exploration strategy. However, censoring is required to ensure that our BO method does not unnecessarily explore the incomplete queries or queries near them. 
Moreover, we can reduce the effect of this aggressive exploration by increasing the value of $m$, since a larger $m$ leads to less aggressive censoring. Another limitation of the censoring method is that it needs to know the minimum function value. However, when the minimum value is unknown, we can use a suitable lower bound as a proxy for the minimum function value. However, regret analysis of such a method can be challenging and left for future research.

%% file: latex/batch_bo.tex

Notably, our BO-SDF problem subsumes batch BO as a special case.
Specifically, the function queries in batch BO can be viewed as being sequentially selected, in which some queries need to be selected while some other queries of the same batch are still \emph{incomplete} \cite{desautels2014parallelizing}.
Therefore, in batch BO with a batch size of $\cB_x$, the incomplete queries can be viewed as delayed feedback with fixed delays s.t.~$d_s \leq \cB_x-1$. In this case, by choosing $m=\cB_x-1$, we can ensure that $\rho_m = \bP(d_s \leq \cB_x-1)=1$.
As a result, our method of censoring (Section~\ref{subsec:gp:and:censored:feedback}) gives a better treatment to the incomplete queries than the classic technique from batch BO, which we will demonstrate next via both intuitive justification and regret comparison.

The classic technique to handle the incomplete queries in batch BO is \emph{hallucination} which was proposed by the GP-BUCB algorithm from \citep{desautels2014parallelizing}. It has also been adopted by a number of its extensions \citep{daxberger2017distributed,chowdhury2019batch}.
In particular, the feedback of incomplete queries is hallucinated to be (i.e., censored with) the GP posterior mean computed using only the completed queries (excluding the incomplete queries). It is equivalent to keeping the GP posterior mean unchanged (i.e., unaffected by the incomplete queries) while updating the GP posterior variance using all selected queries (including the incomplete queries).
Note that in contrast to the hallucination, our censoring technique sets the incomplete observations to the minimum function value (i.e., $0$) (Section~\ref{subsec:gp:and:censored:feedback}).
As a result, compared with the hallucination where the GP posterior mean is unaffected by the incomplete queries, our censoring can \emph{reduce the value of the GP posterior mean at the incomplete queries} (because we have treated their observations as if they are $0$). As a result, our censoring can \emph{discourage these incomplete queries from being unnecessarily queried again}. It encourages the algorithm to explore other unexplored regions, hence leading to better exploration.

Interestingly, the better exploration resulting from our censoring technique also translates to a tighter regret upper bound.
Specifically, in batch BO with a batch size of $\cB_x$, after plugging in $m=\cB_x-1$ and $\rho_m=1$, our next result gives the regret upper bound of \ref{alg:GP-UCB-SDF}.
\begin{restatable}{prop}{ucbRegretBatchBO}
    \label{prop:ucbRegretBatchBO}
    With probability at least $1-\delta$,
    \eqs{
        \Regret_T=\tilde{O}\Lp \gamma_T \sqrt{T} + \cB_x\gamma_T \Rp.
    }
\end{restatable}
As long as $\gamma_T=\Omega(\log T)$ which holds for most commonly used kernels such as the squared exponential (SE) and Mat\'ern kernels, both terms in \cref{prop:ucbRegretBatchBO} grow slower than $\Regret_T=\tilde{O}(\gamma_T\sqrt{T}\exp(\gamma_{\cB_x-1}))$ which is the regret upper bound of the GP-BUCB algorithm based on hallucination~\cite{desautels2014parallelizing}.
Importantly, to achieve a sub-linear regret upper bound (i.e., to ensure that $\exp(\gamma_{\cB_x-1})$ can be upper-bounded by a constant), GP-BUCB \citep{desautels2014parallelizing} and its extensions \citep{daxberger2017distributed,kandasamy2018parallelised,chowdhury2019batch} all require a special initialization scheme (i.e., uncertainty sampling). It is often found unnecessary in practice \citep{kandasamy2018parallelised} and hence represents a gap between theory and practice.
Interestingly, our tighter regret upper bound from \cref{prop:ucbRegretBatchBO} has removed the dependence on $\exp(\gamma_{\cB_x-1})$ and hence eliminated the need for uncertainty sampling, thereby closing this gap between theory and practice.

Note that our discussions above regarding the advantage of censoring over hallucination also apply to TS-based algorithms. In particular, in batch BO with a batch size of $\cB_x$, the regret upper bound of \ref{alg:GP-TS-SDF} is given by the following result.
\begin{restatable}{prop}{tsRegretBatchBO}
    \label{prop:tsRegretBatchBO}
    With probability at least $1-\delta$,
    \eqs{
        \Regret_T=\tilde{O}\Lp \sqrt{T\gamma_T} (\sqrt{\gamma_T} + 1) + \cB_x \gamma_T + \cB_x\sqrt{T} \Rp.
    }
\end{restatable}
All three terms in \cref{prop:tsRegretBatchBO} grow slower than $\Regret_T=\tilde{O}(\exp(\gamma_{\cB_x-1})\sqrt{T\gamma_T}(\sqrt{\gamma_T}+1))$ which is the regret upper bound of the GP-BTS algorithm based on hallucination \citep{chowdhury2019batch} as long as $\gamma_T=\Omega(\log T)$.

To summarize, though batch BO is only a special case of our BO-SDF problem, we have made non-trivial contributions to the algorithm and theory of batch BO. Our \ref{alg:GP-UCB-SDF} and \ref{alg:GP-TS-SDF} algorithms and our theoretical analyses may serve as inspiration for future works on batch BO.

%% file: latex/contx_bandits.tex

Before making a decision, sometimes additional information is available to the learner in many real-life scenarios (e.g., users' profile information for an online platform and patients' medical history before clinical trials). This additional information is referred to as \emph{context} in the literature, and the value of a function depends on the context \citep{ICML13_agrawal2013thompson,AISTATS11_chu2011contextual,NIPS11_krause2011contextual,WWW10_li2010contextual}. The learner can design the algorithms to exploit the contextual information to make a better decision. We can extend our results for the BO-SDF problem to the contextual Gaussian process bandits \citep{NIPS11_krause2011contextual} with stochastic delayed feedback where a non-linear function maps a context to the feedback.

Let $\cQ \subset \R^n$ be a finite set of available actions (or queries) and $\cZ \subset \R^{n^\prime}$ be a set of all contexts where $n, n^\prime \ge 1$.
In round $t$, the environment generates a context $z_t \in \cZ$. Then, the learner selects an action $x_t \in \cQ$ and observes noisy feedback denoted by $y_t = g(z_t, x_t) + \epsilon_t$. 
We assume that $g: \cZ \times \cQ \rightarrow \R$ and $\epsilon_t$ is an $R$-sub-Gaussian noise. The $y_t$ is only observed after a random delay $d_t$, which is generated from an unknown distribution $\cD$.

We refer to this new problem as \emph{`Contextual Gaussian process Bandits under Stochastic Delayed Feedback'} (CGB-SDF). 
The unknown function $g$, context space $\cZ$, action space $\cQ$, and unknown delay distribution $\cD$ identify an instance of the CGB-SDF problem. 
The optimal action $x_t^\star$ for a given context $z_t$ is the action where the function $g$ has its maximum value, i.e., $x_t^\star = \argmax_{x \in \cQ} g(z_t, x)$.
The learner incurs a penalty of $\left(g(z_t, x_t^\star) - g(z_t, x_t)\right)$ for a context $z_t$ and action $x_t$. 
We aim to learn a policy that achieves the minimum cumulative penalties or \emph{regret} which is given by $\Regret_T = \sum_{t=1}^T \left(g(z_t, x_t^\star) - g(z_t, x_t)\right)$.

Our goal is to learn a policy that has a small sub-linear regret, i.e., $\lim_{T \rightarrow \infty}{\Regret_T}/T \rightarrow 0$. 
The sub-linear regret here implies that the policy will eventually select the optimal action for a given context. 
We have adapted our algorithms for the BO-SDF problem to the CGB-SDF problem and shown that they also enjoy similar sub-linear regret guarantees; see Appendix~\ref{app:sec:contextual:gp:bandit} for more details.

%% file: latex/experiment.tex

We compare with previous baseline methods that can be applied to BO-SDF problems after modifications. Firstly, we compare with standard GP-UCB~\cite{ICML10_srinival2010gaussian} which ignores the delayed observations when selecting the next query. Note that GP-UCB is likely to repeat previously selected queries in some iterations when the GP posterior remains unchanged (i.e., if no observation is newly collected between two iterations). Next, we also compare with the batch BO method of \emph{asy-TS}~\cite{kandasamy2018parallelised} which, similarly to GP-UCB, ignores the delayed observations when using TS to choose the next query. Lastly, we compare with the batch BO methods of \emph{GP-BUCB}~\cite{desautels2014parallelizing} and \emph{GP-BTS}~\cite{chowdhury2019batch} which handle the delayed observations via hallucination (Section~\ref{sec:batch_bo}). Following the suggestion from~\cite{ICML20_vernade2020linear}, if not specified otherwise, we set $m$ to be $2\mu$ where $\mu$ is the mean of the delay distribution. However, this is for convenience only since we will demonstrate (Section~\ref{subsec:exp:synth}) that our algorithms consistently achieve competitive performances as long as $m$ is large enough. Therefore, $\mu$ does not need to be known in practice. We defer some experimental details to Appendix~\ref{app:sec:more:experimenta:detail} due to space constraint. Our code is released at \url{https://github.com/daizhongxiang/BO-SDF}.

\begin{figure*}
	\centering
	\begin{tabular}{cccc}
		\hspace{-3mm}
		\includegraphics[width=0.25\linewidth]{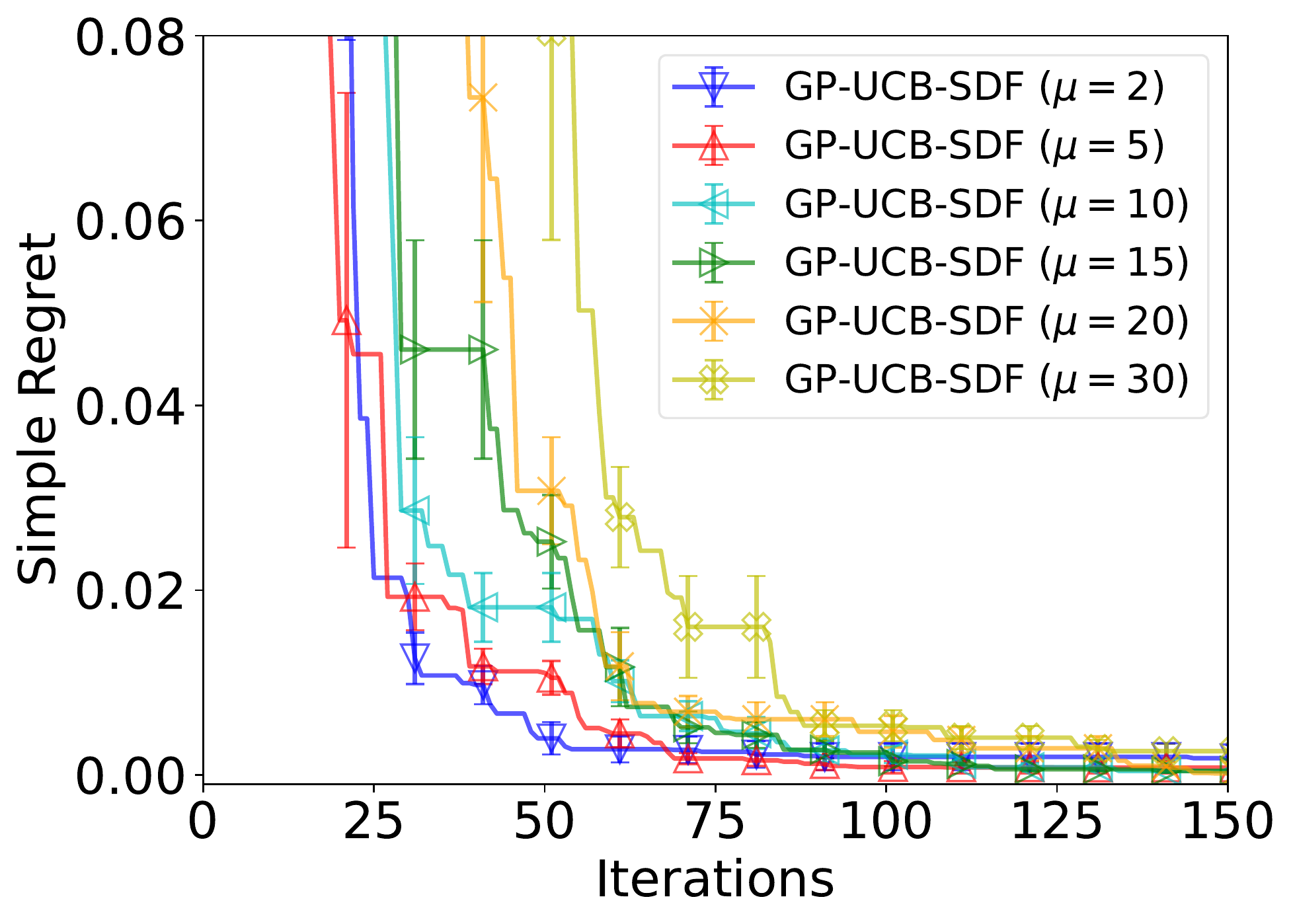}& \hspace{-5mm} 
		\includegraphics[width=0.25\linewidth]{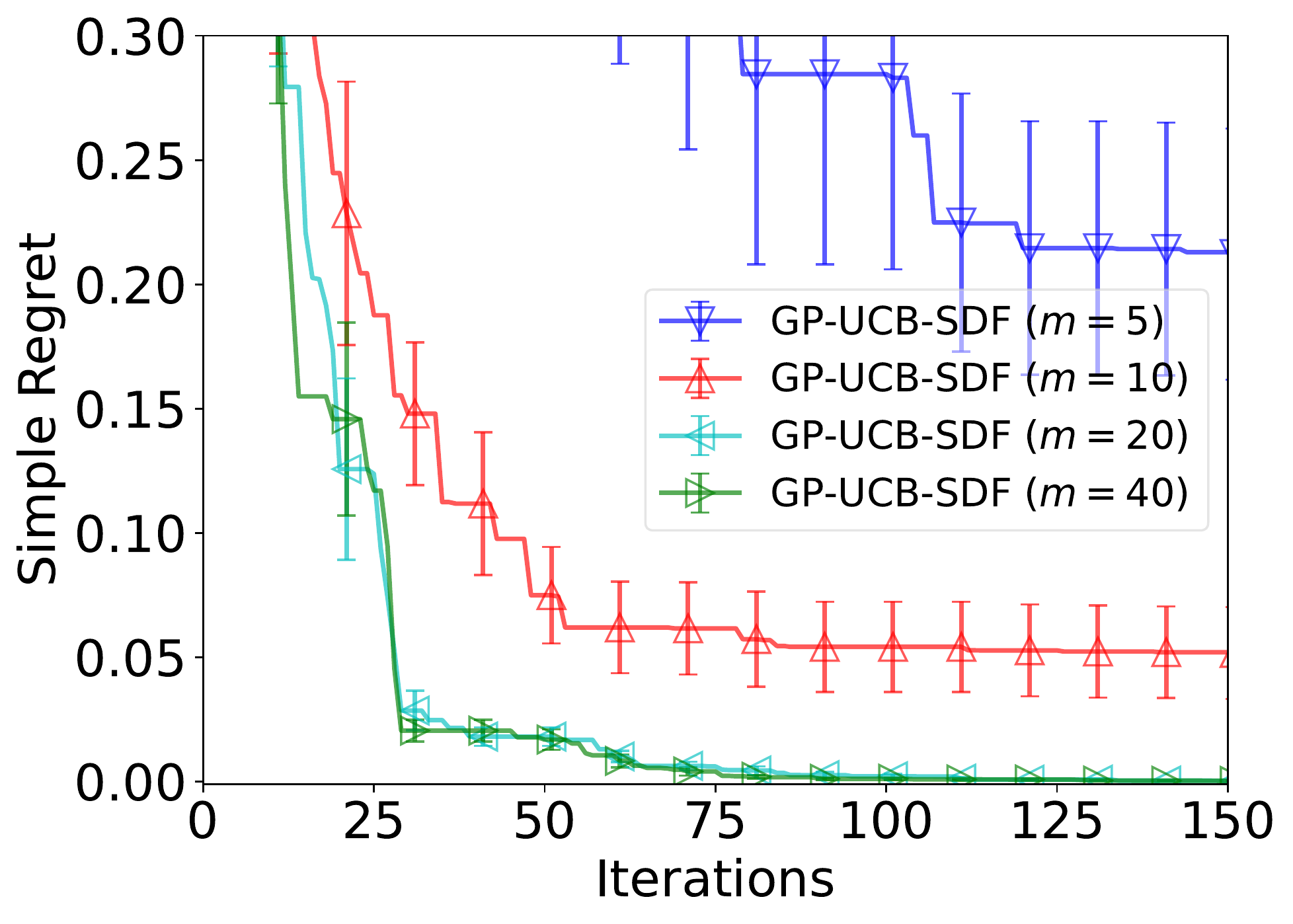}& \hspace{-5mm} 
		\includegraphics[width=0.25\linewidth]{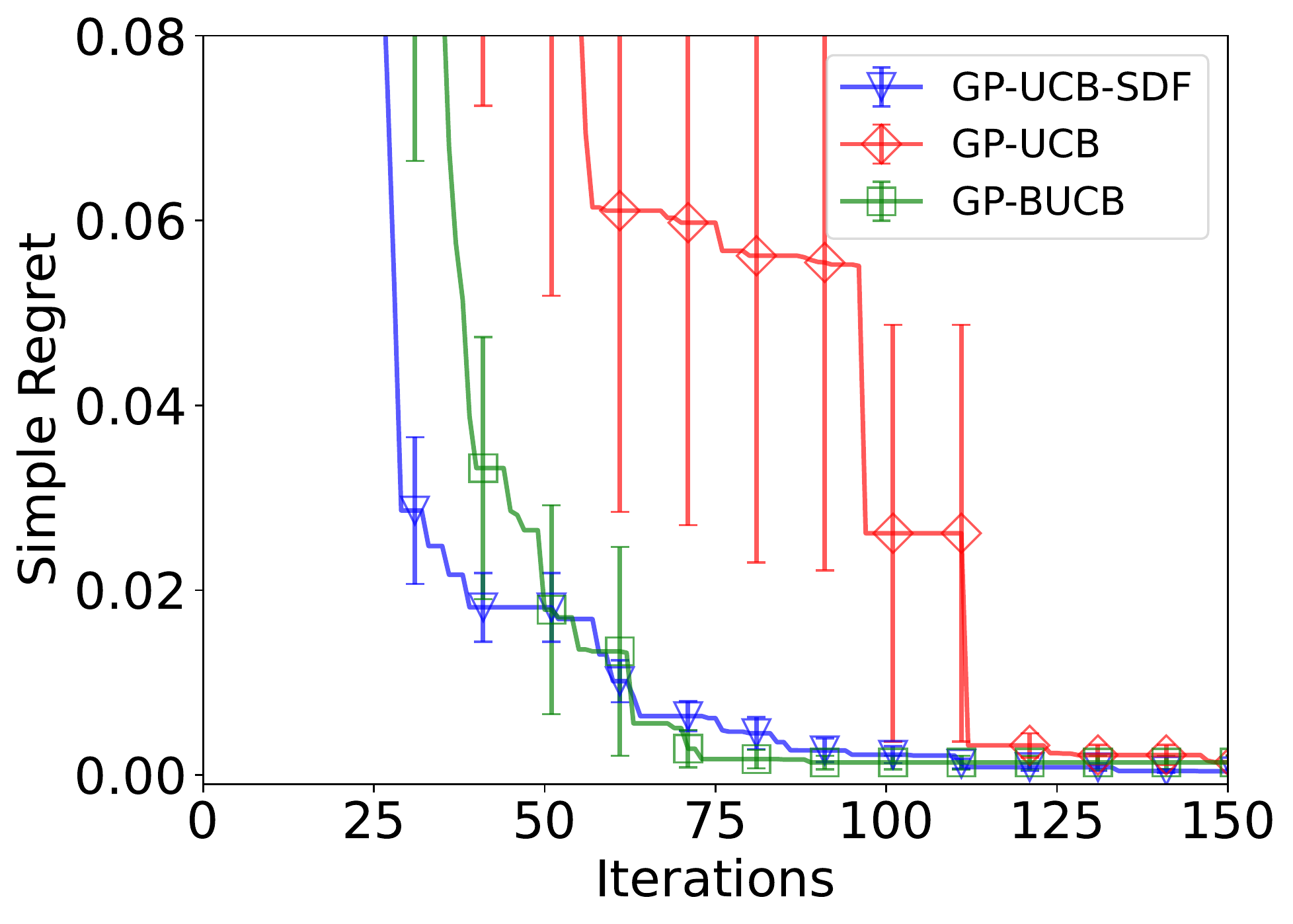}& \hspace{-5mm}
		\includegraphics[width=0.25\linewidth]{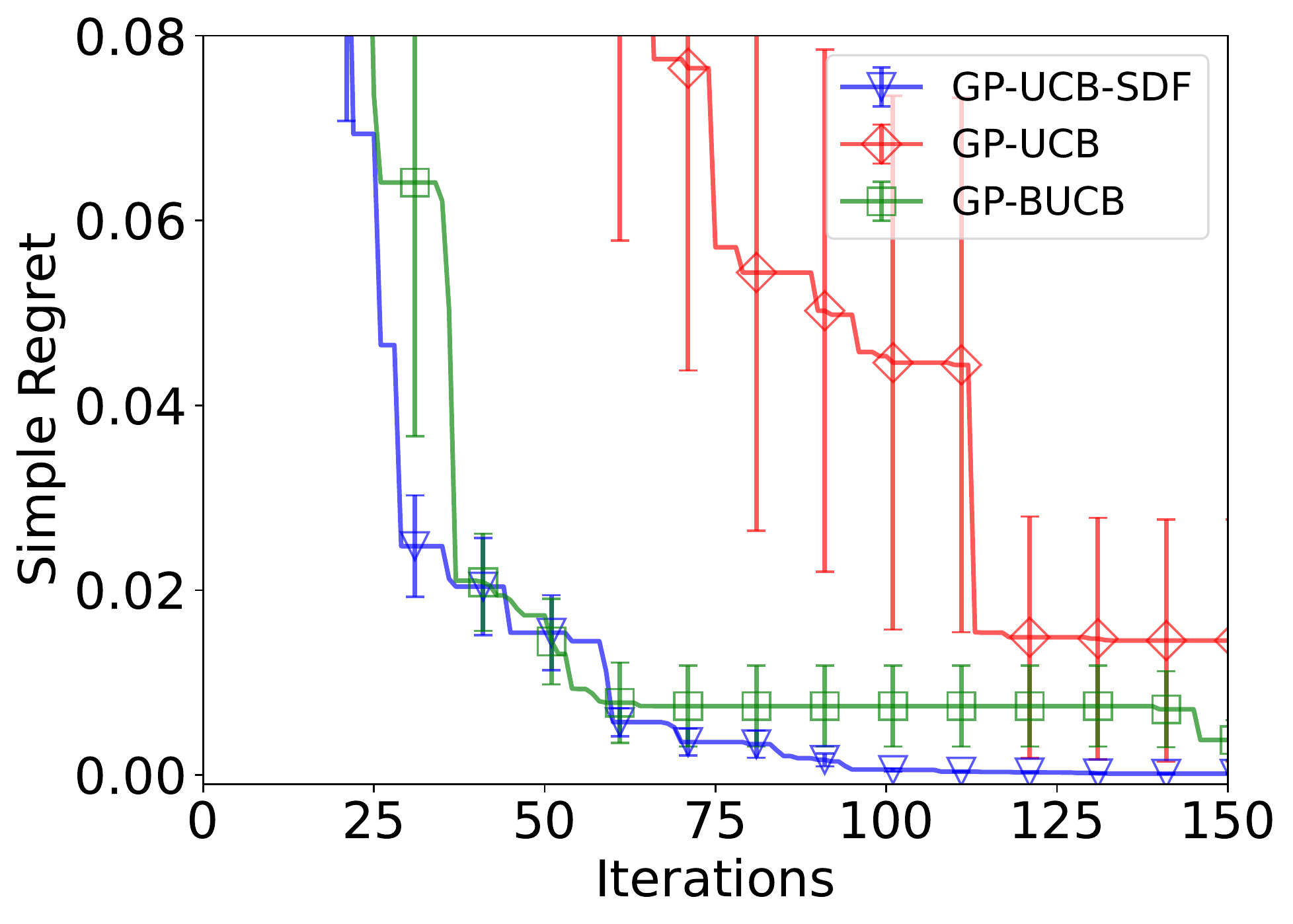}\\
		{(a)} & {(b)} & {(c)} & {(d)}
	\end{tabular}
	\caption{
		Performances of our \ref{alg:GP-UCB-SDF} for (a) different delay distributions and (b) different values of $m$ ($\mu$ is fixed at $\mu=10$). 
		Performances of \ref{alg:GP-UCB-SDF} and baseline methods for (c) stochastic and (d) deterministic delay distributions.
	}
	\label{fig:exp:synth}
\end{figure*}

\begin{figure*}[t]
     \centering
     \vspace{3mm}
     \begin{tabular}{cccc}
         \hspace{-3mm}
         \includegraphics[width=0.25\linewidth]{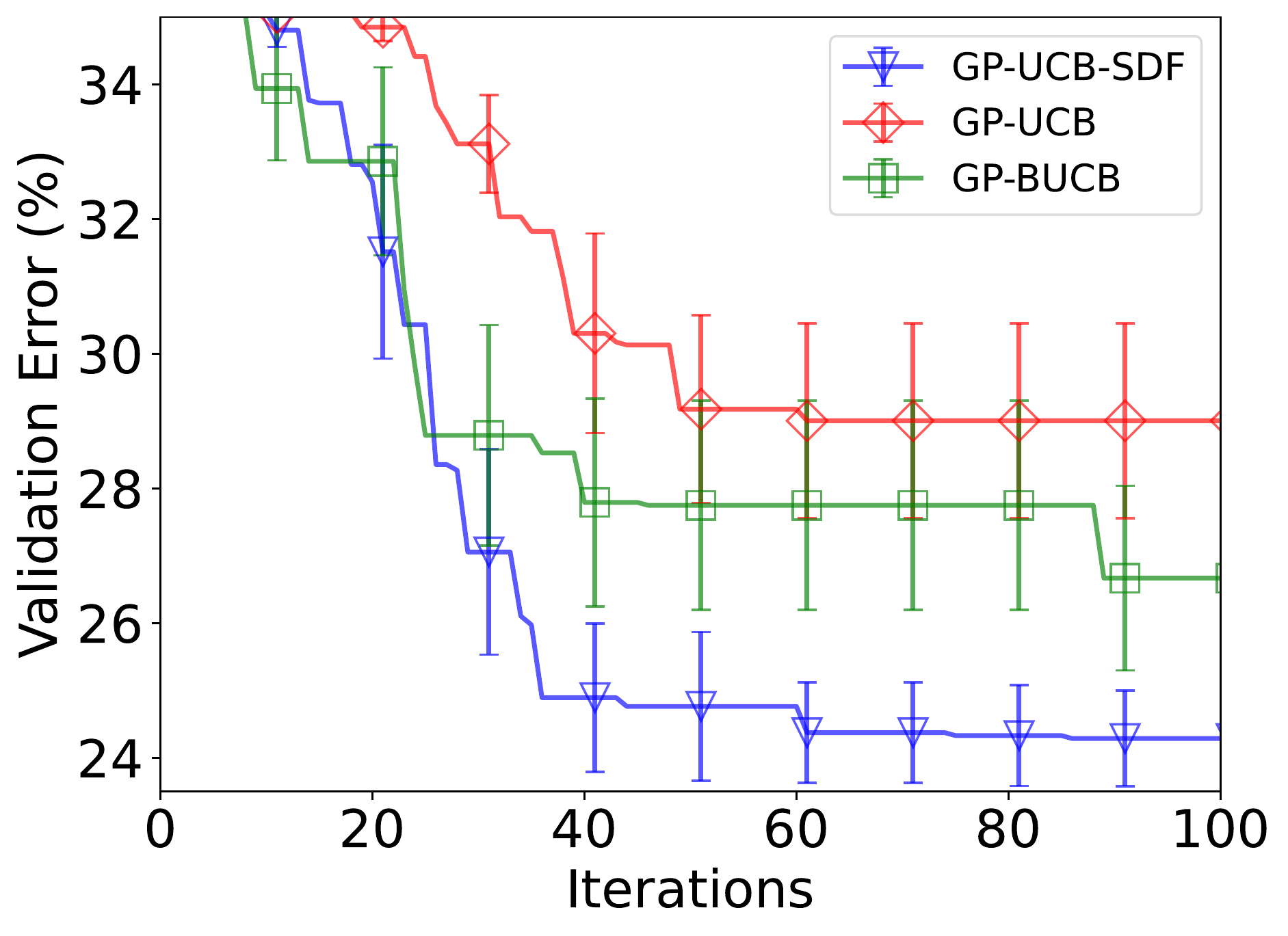}& \hspace{-5mm} 
         \includegraphics[width=0.25\linewidth]{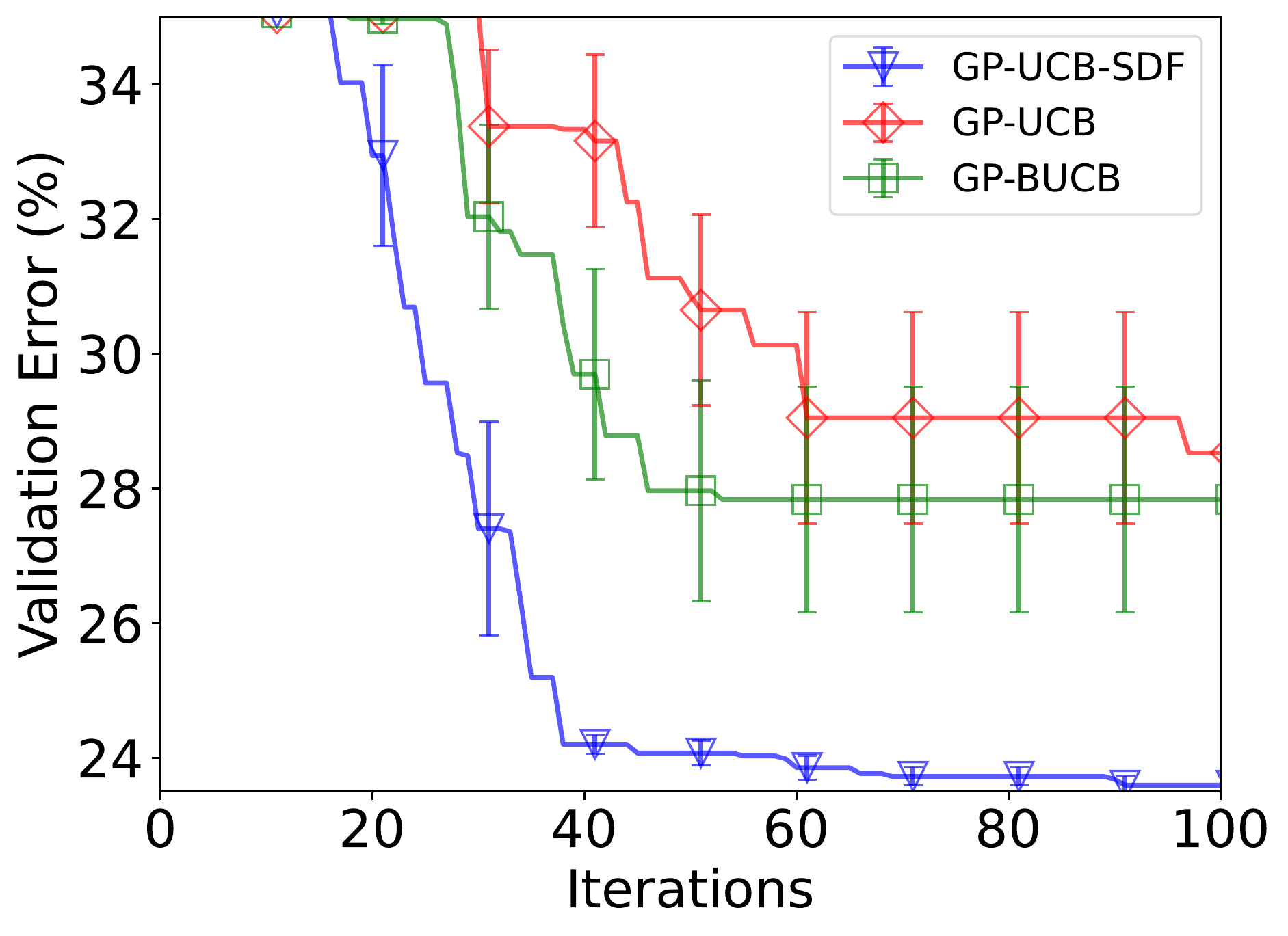}& \hspace{-5mm} 
         \includegraphics[width=0.25\linewidth]{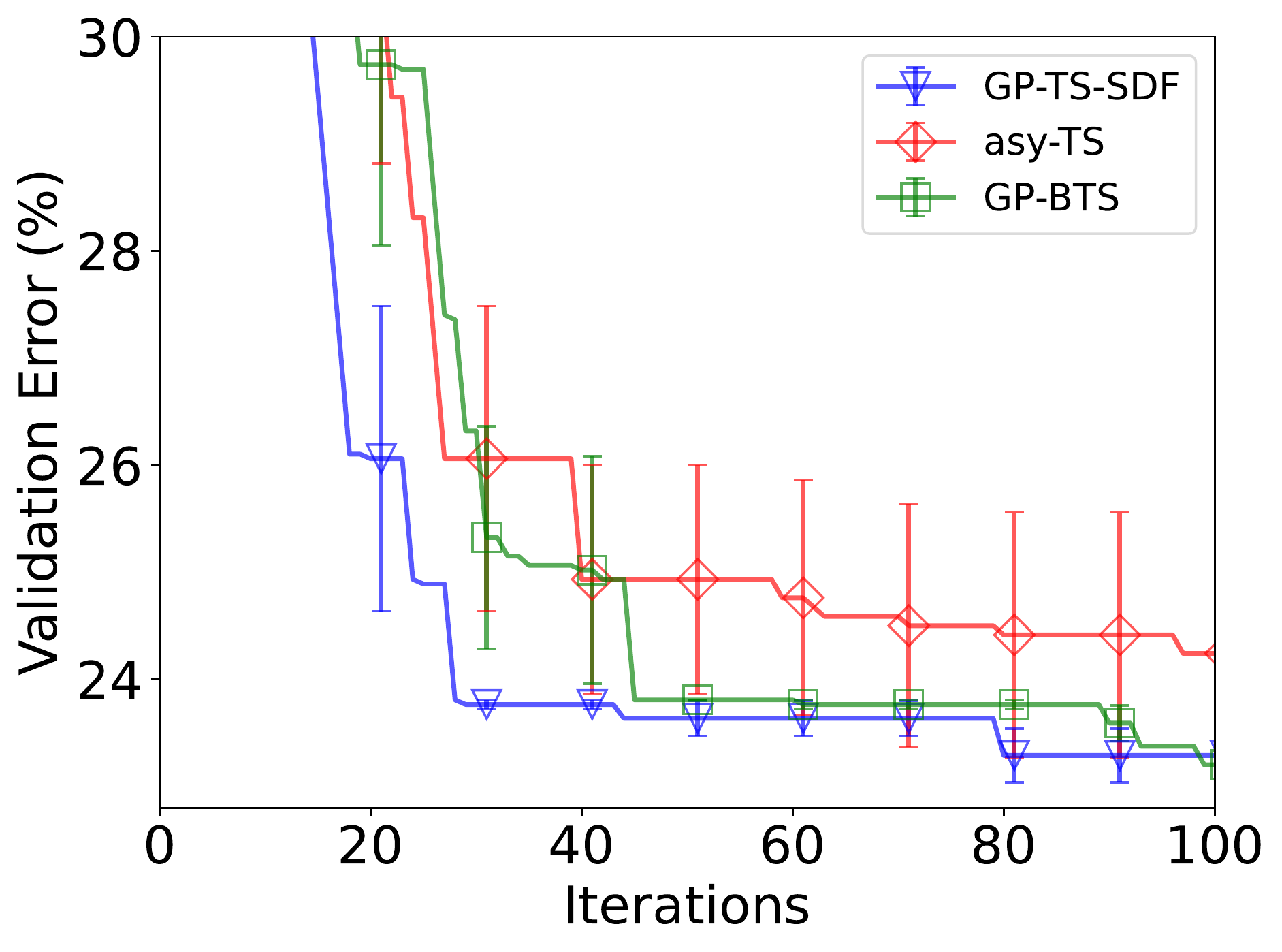}& \hspace{-5mm} 
         \includegraphics[width=0.25\linewidth]{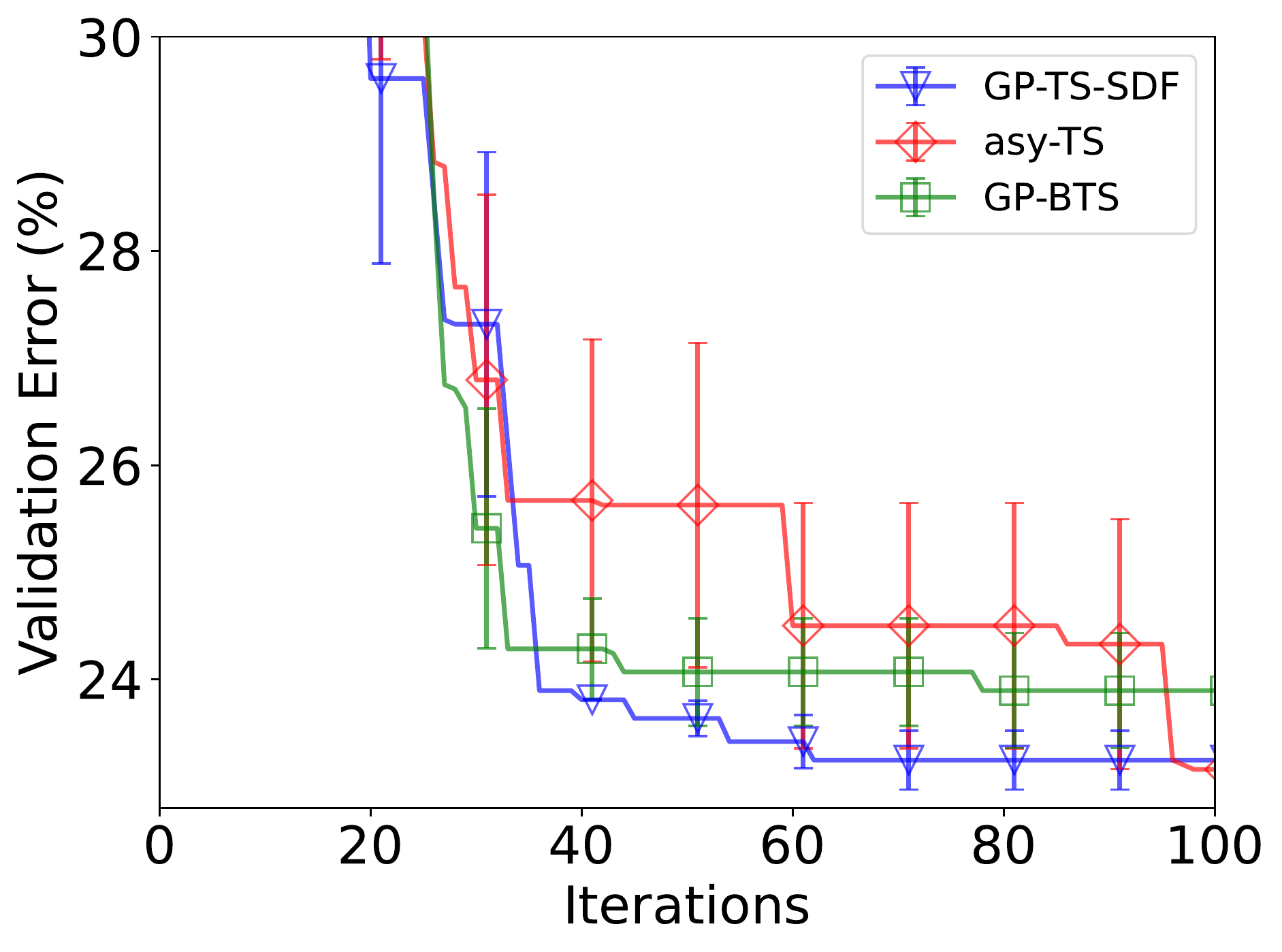}\\
         {(a)} & {(b)} & {(c)} & {(d)}
         \\~\\
         \hspace{-3mm}
         \includegraphics[width=0.25\linewidth]{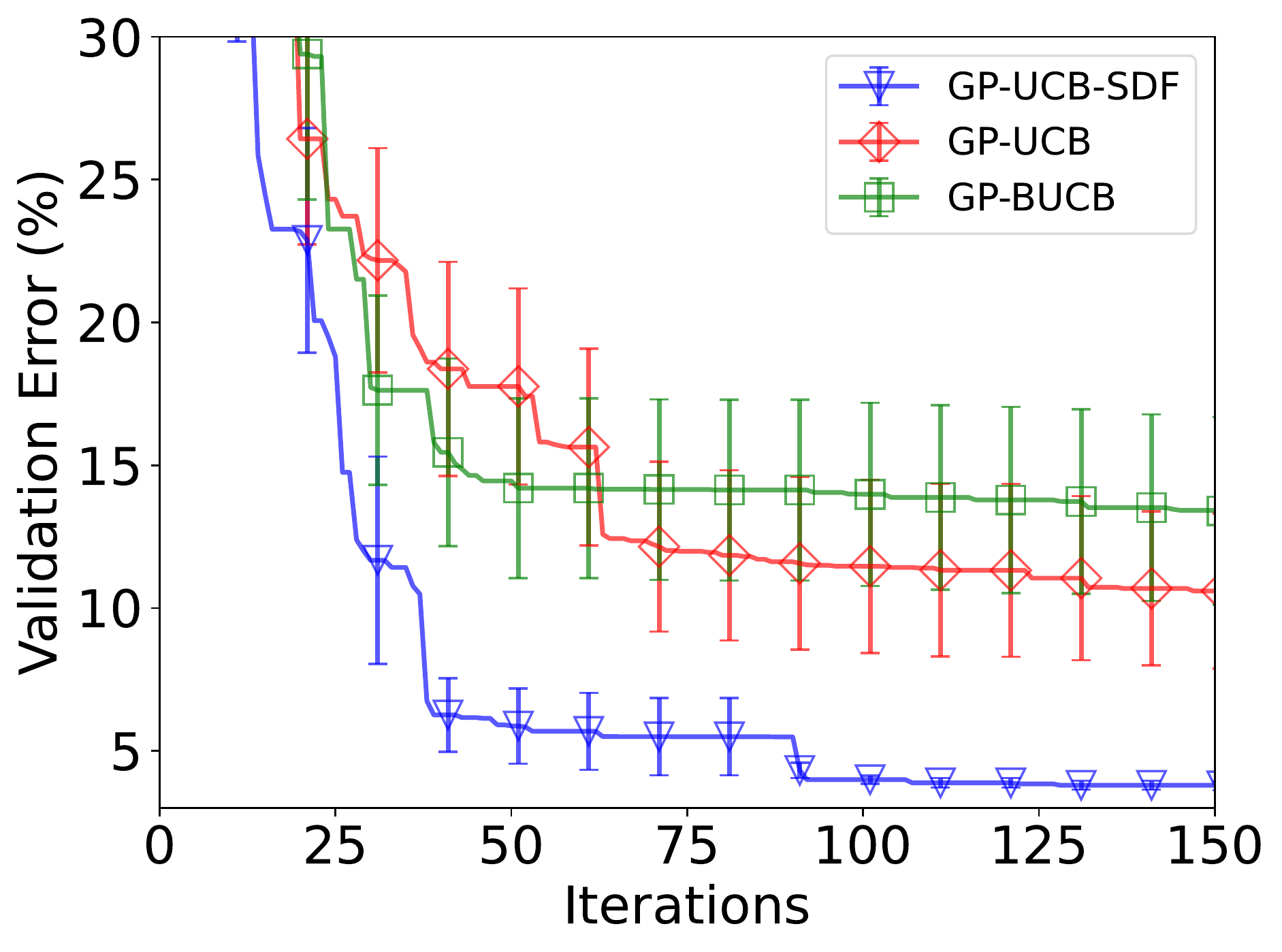}& \hspace{-5mm} 
         \includegraphics[width=0.25\linewidth]{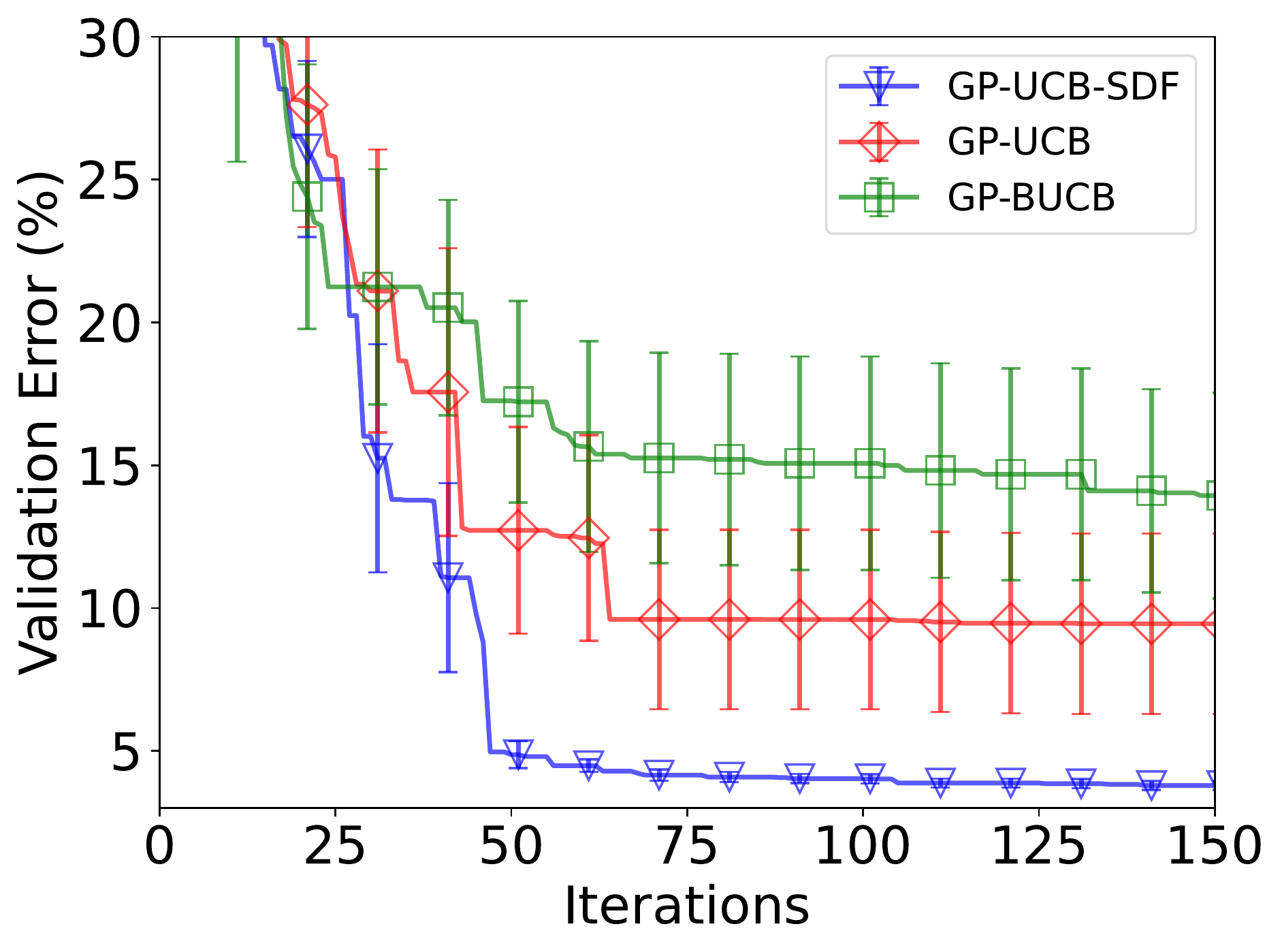}& \hspace{-5mm} 
         \includegraphics[width=0.25\linewidth]{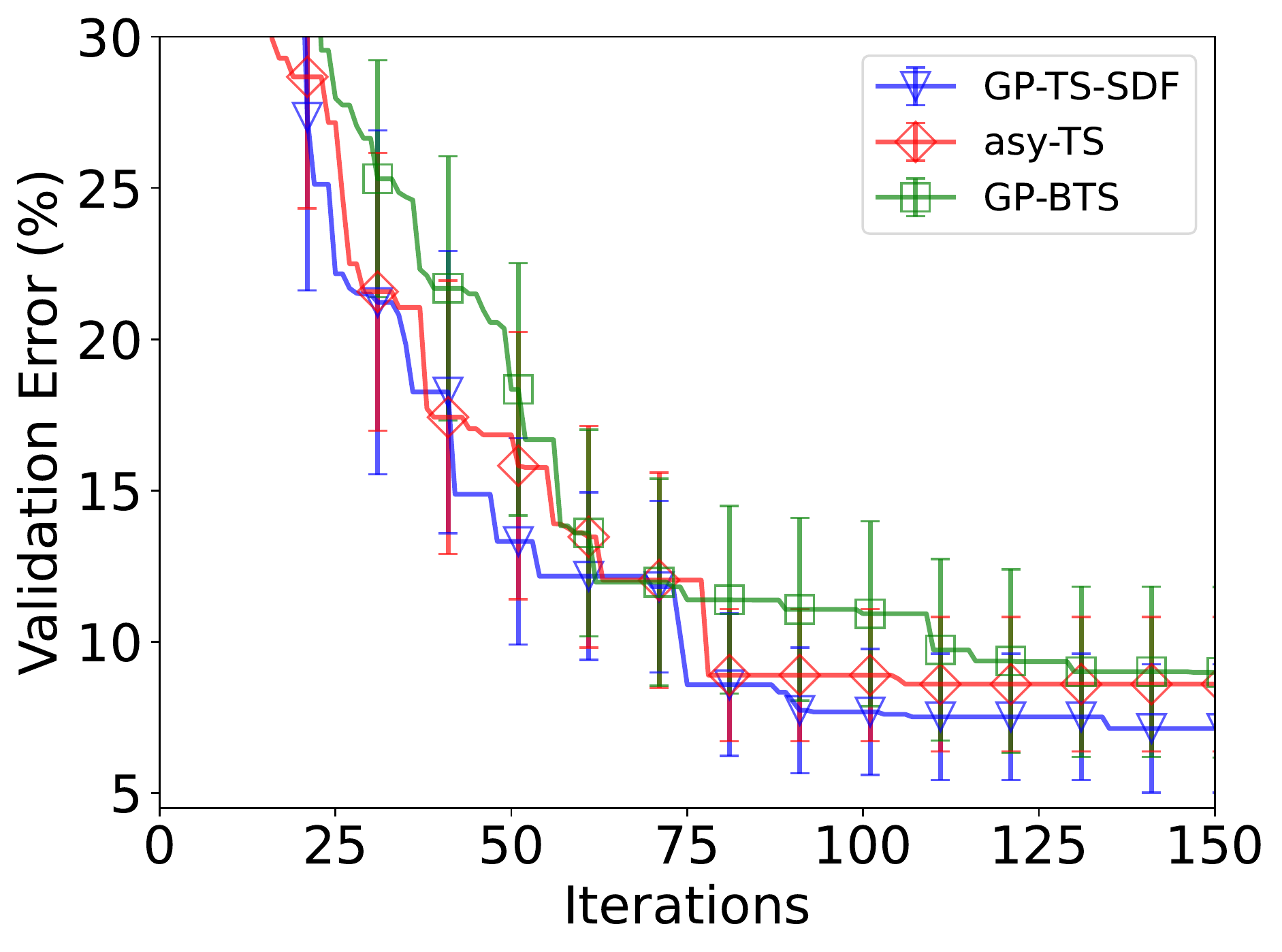}& \hspace{-5mm} 
         \includegraphics[width=0.25\linewidth]{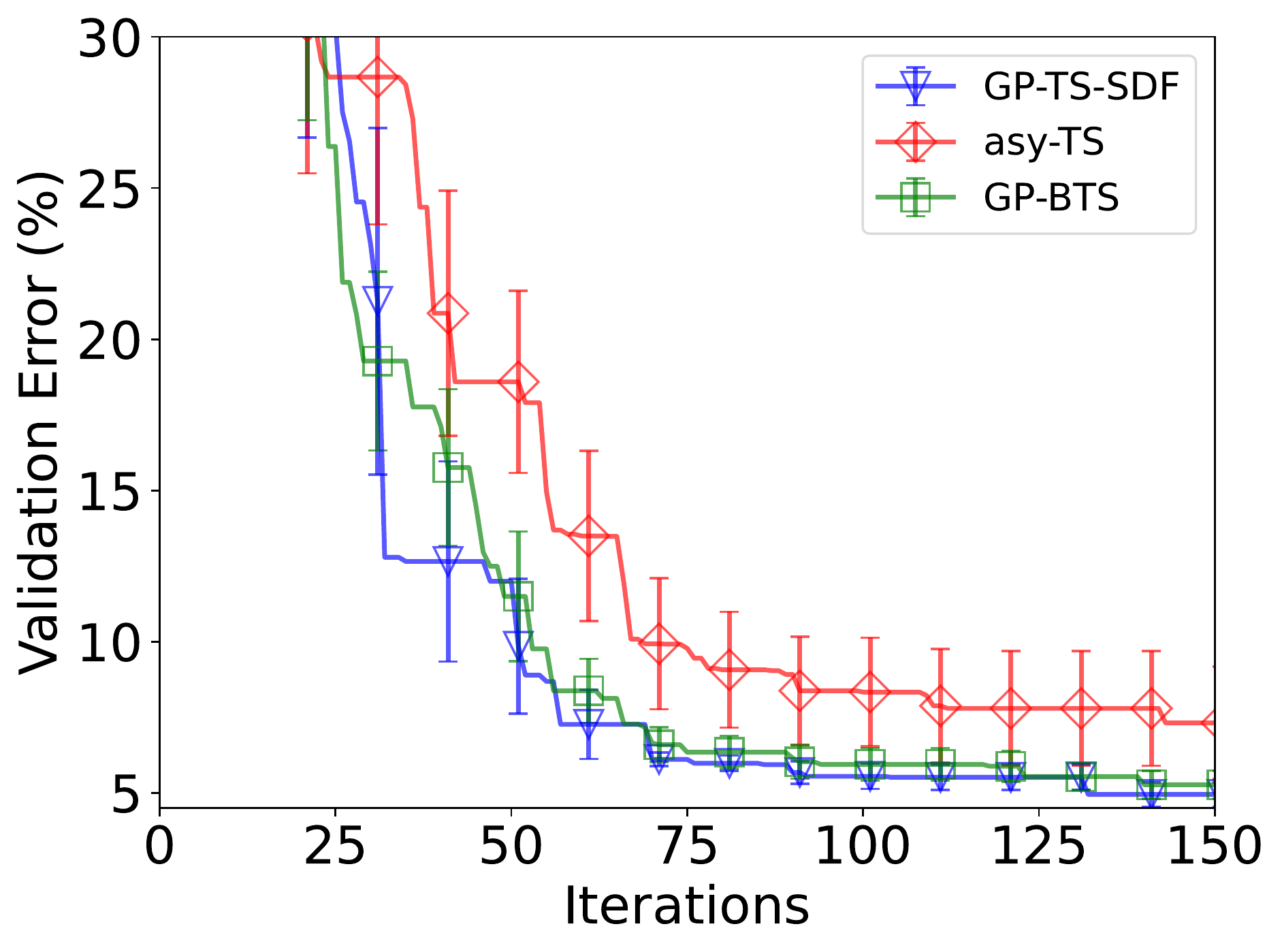}\\
         {(e)} & {(f)} & {(g)} & {(h)}
     \end{tabular}
     \caption{
        Performances for hyperparameter tuning of SVM: UBC-based methods with (a) stochastic and (b) deterministic delay distributions, and TS-based methods with (c) stochastic and (d) deterministic delay distributions.
        Performances for hyperparameter tuning of CNN: UBC-based methods with (e) stochastic and (f) deterministic delay distributions, and TS-based methods with (g) stochastic and (h) deterministic delay distributions.
     }
     \label{fig:exp:automl}
\end{figure*}

\subsection{Synthetic Experiments}
\label{subsec:exp:synth}
For this experiment, we use a Poisson distribution (with a mean parameter $\mu$) as the delay distribution and sample the objective function $f$ from a GP with an SE kernel defined on a discrete $1$-dimensional domain. Fig.~\ref{fig:exp:synth}a plots the performances of our \ref{alg:GP-UCB-SDF} for different delay distributions and shows that larger delays (i.e., larger $\mu$'s) result in worse performances. Fig.~\ref{fig:exp:synth}b shows that for a fixed delay distribution (i.e., $\mu=10$), an overly small value of $m$ leads to large regrets since it causes our algorithm to ignore a large number of delayed observations (i.e., overly aggressive censoring). On the other hand, a large $m$ is desirable since $m=20=2\mu$ and $m=40$ produce comparable performances.

However, an overly large value of $m$ may exert an excessive resource requirement since we may need to cater to many pending function evaluations. Therefore, when the delay distribution (hence $\mu$) may be unknown, we recommend using a large value of $m$ that is allowed by the resource constraints. We now compare the performances of our \ref{alg:GP-UCB-SDF} with the baselines of GP-UCB and GP-BUCB for both stochastic ($\mu=10$, Fig.~\ref{fig:exp:synth}c) and deterministic (delays fixed at $10$, Fig.~\ref{fig:exp:synth}d) delay distributions. The figures show that our \ref{alg:GP-UCB-SDF} achieves smaller simple regret than the baselines in both settings.

\subsection{Real-world Experiments}
\label{subsec:exp:automl}
In this section, we perform real-world experiments on hyperparameter tuning for ML models under two realistic settings. In the first setting, we consider stochastic delay distributions where the delays are sampled from a Poisson distribution with a mean parameter $\mu=10$. This setting is used to simulate real-world scenarios where the evaluations of different hyperparameter configurations may be stochastically delayed because they may require a different amount of computations (e.g., training a neural network with more layers is more time-consuming), or the machines deployed in the experiment may have different computational capabilities. The second setting considers fixed delay distributions where every delay is fixed at $\cB_x-1=10$, which is used to emulate real-world asynchronous batch BO problems with a batch size of $\cB_x=11$.

\begin{figure*}
     \centering
     \begin{tabular}{cccc}
         \hspace{-3mm}
         \includegraphics[width=0.25\linewidth]{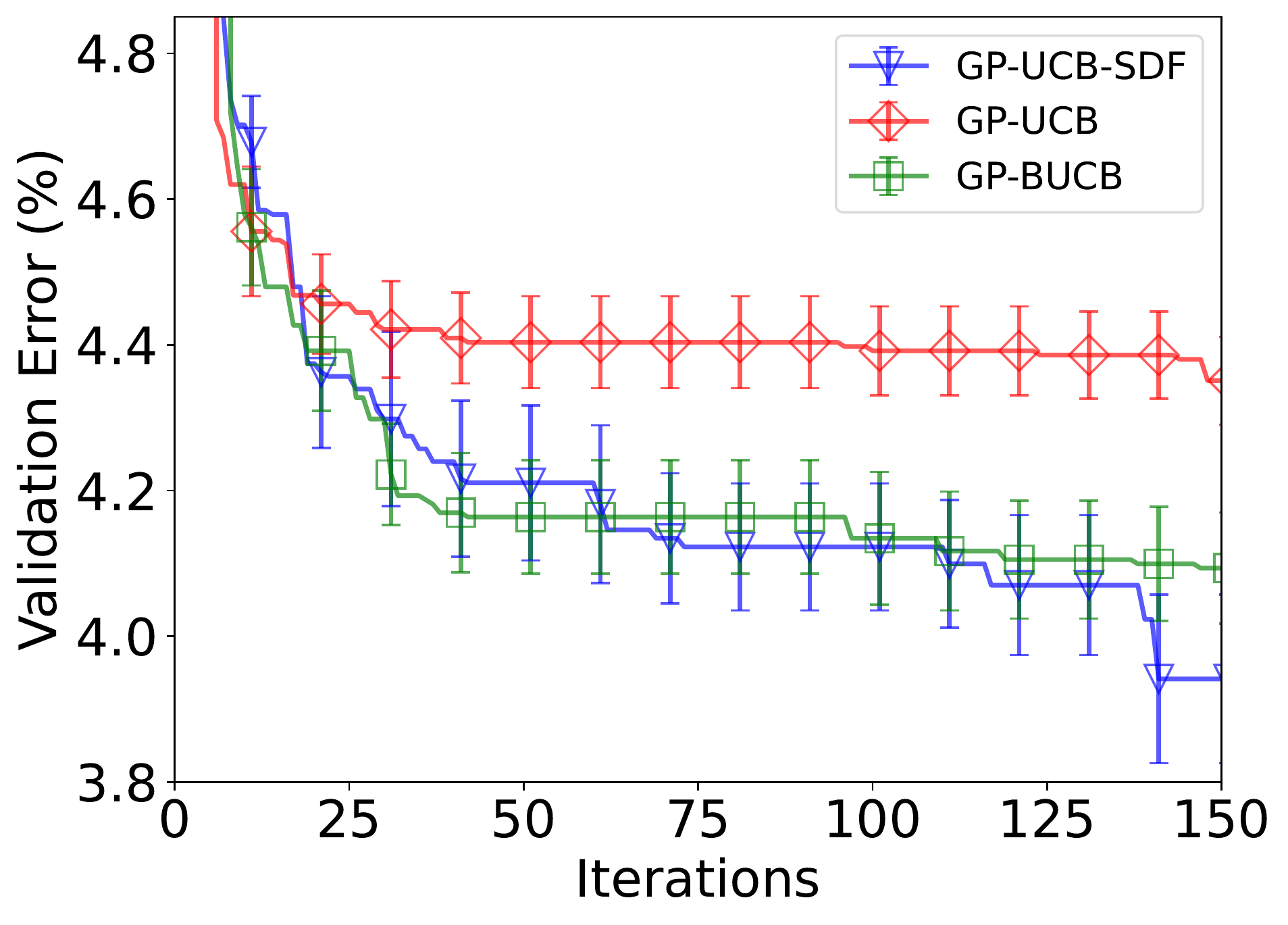}& \hspace{-5mm} 
         \includegraphics[width=0.25\linewidth]{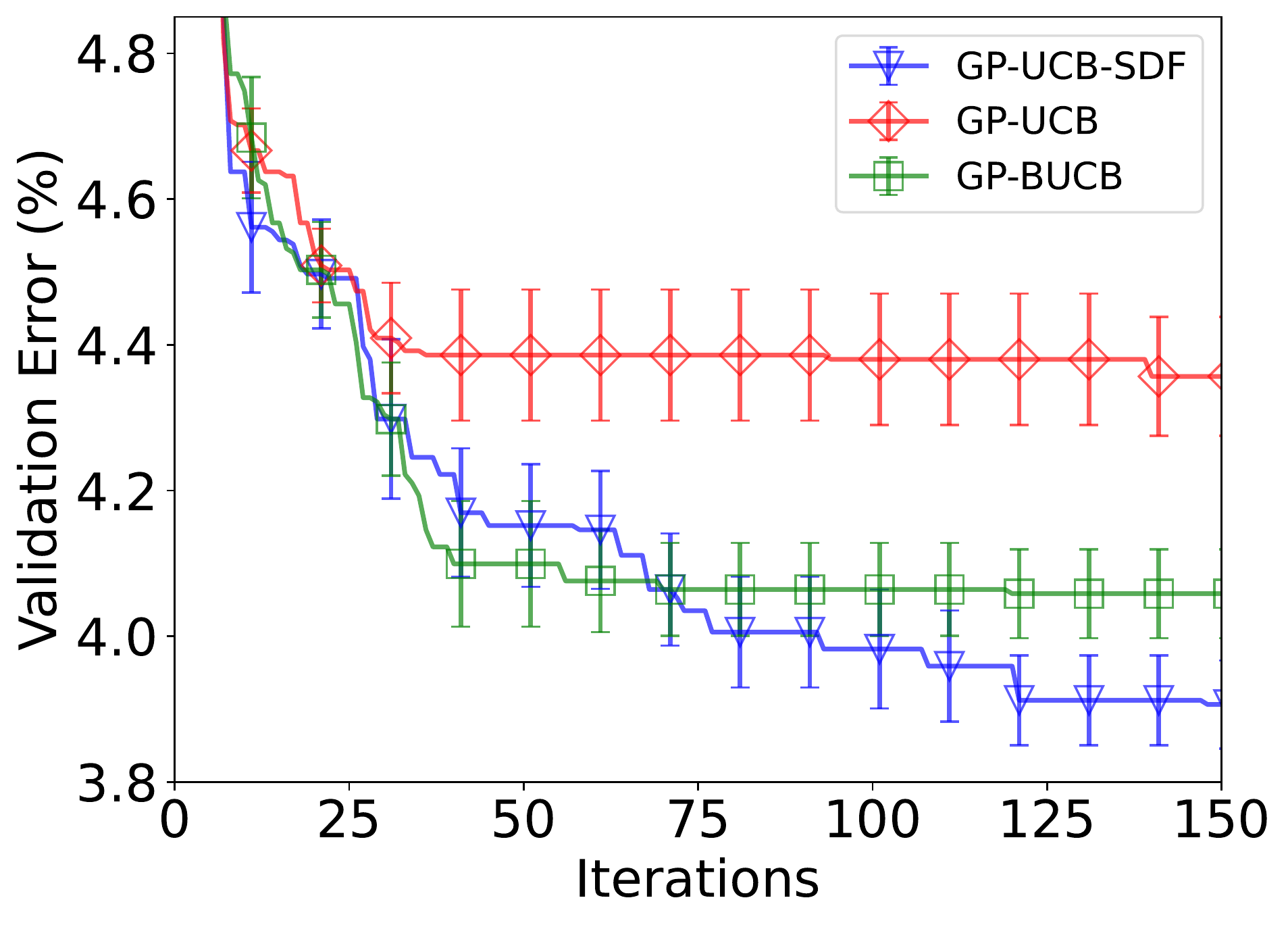}& \hspace{-5mm} 
         \includegraphics[width=0.25\linewidth]{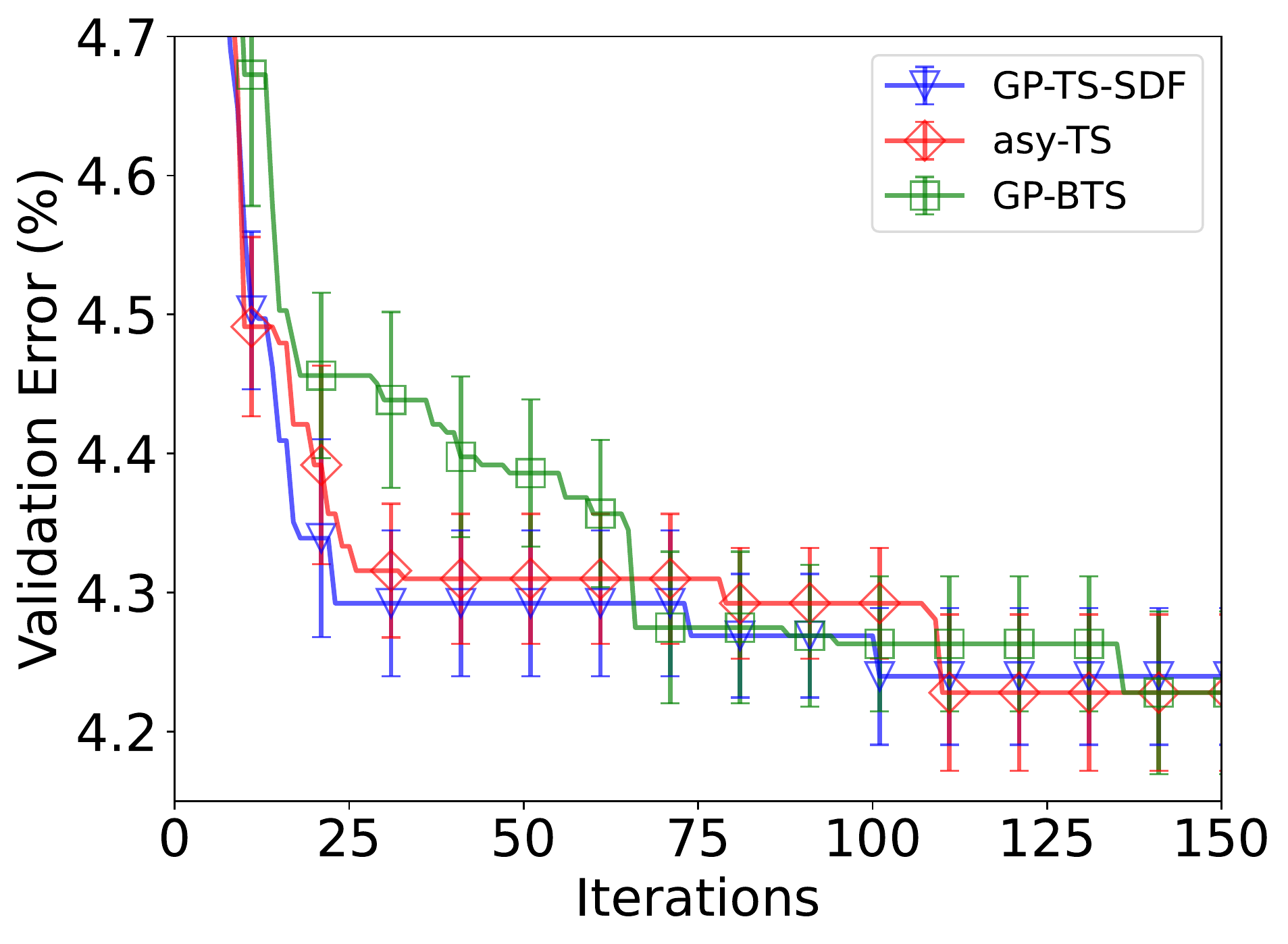}& \hspace{-5mm} 
         \includegraphics[width=0.25\linewidth]{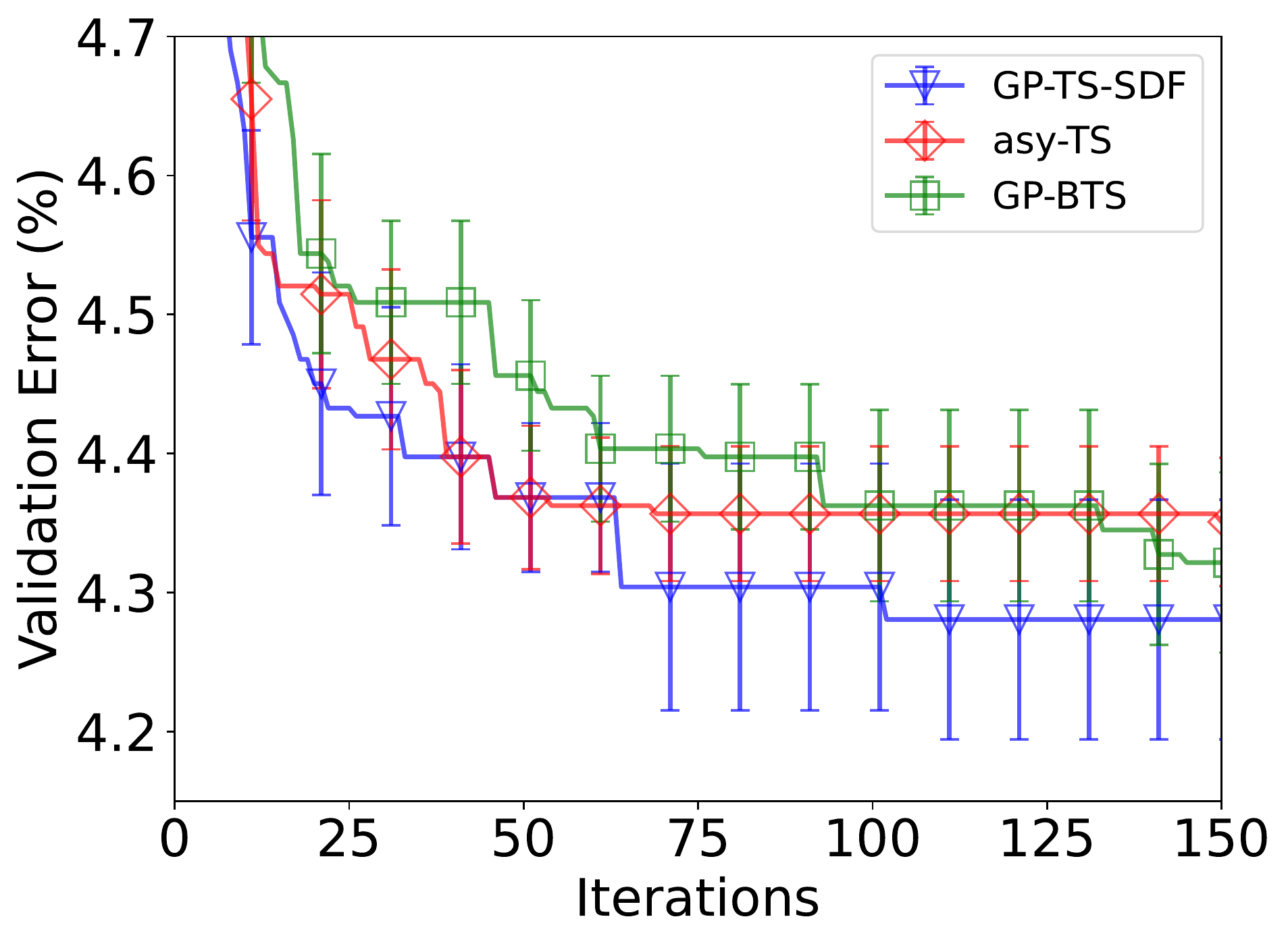}\\
         {(a)} & {(b)} & {(c)} & {(d)}
     \end{tabular}
     \caption{
     Performances for hyperparameter tuning of LR, UBC-based methods for (a) stochastic and (b) deterministic delay distributions, and TS-based methods for (a) stochastic and (b) deterministic delay distributions.
     }
     \label{fig:exp:automl_lr}
\end{figure*}

\begin{figure*}
     \centering
     \begin{tabular}{cccc}
         \hspace{-3mm}
         \includegraphics[width=0.25\linewidth]{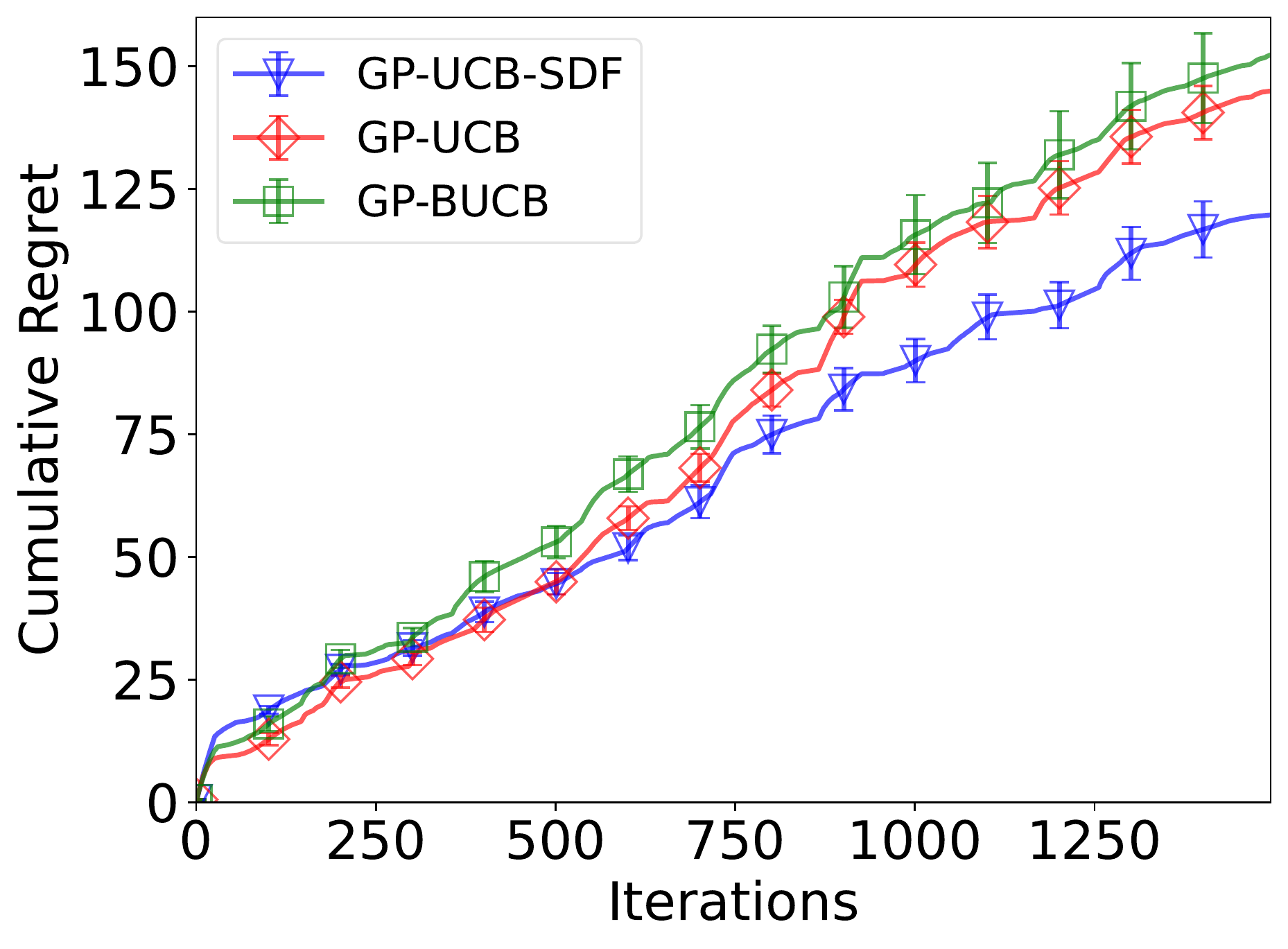}& \hspace{-5mm} 
         \includegraphics[width=0.25\linewidth]{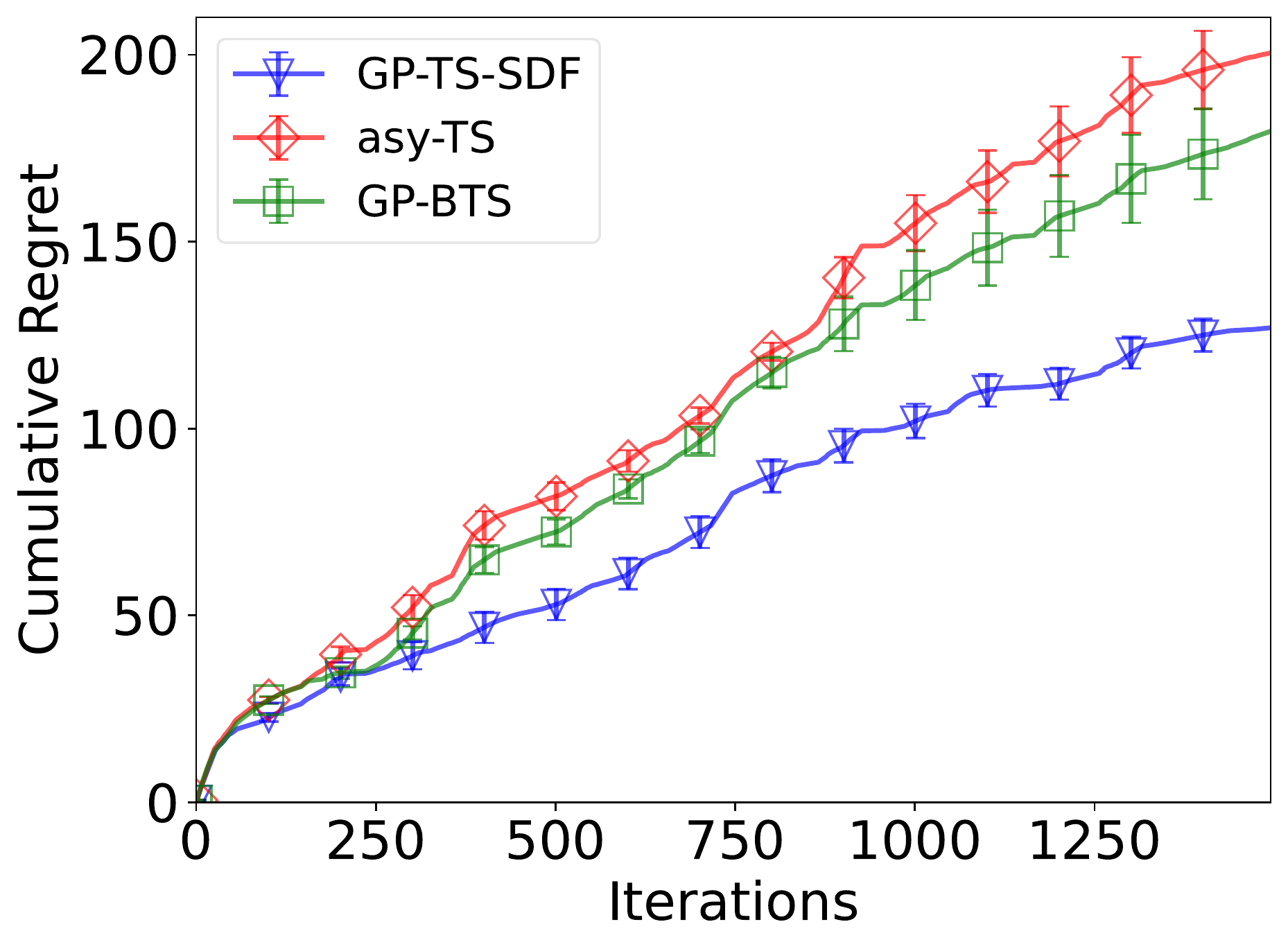}& \hspace{-5mm} 
         \includegraphics[width=0.25\linewidth]{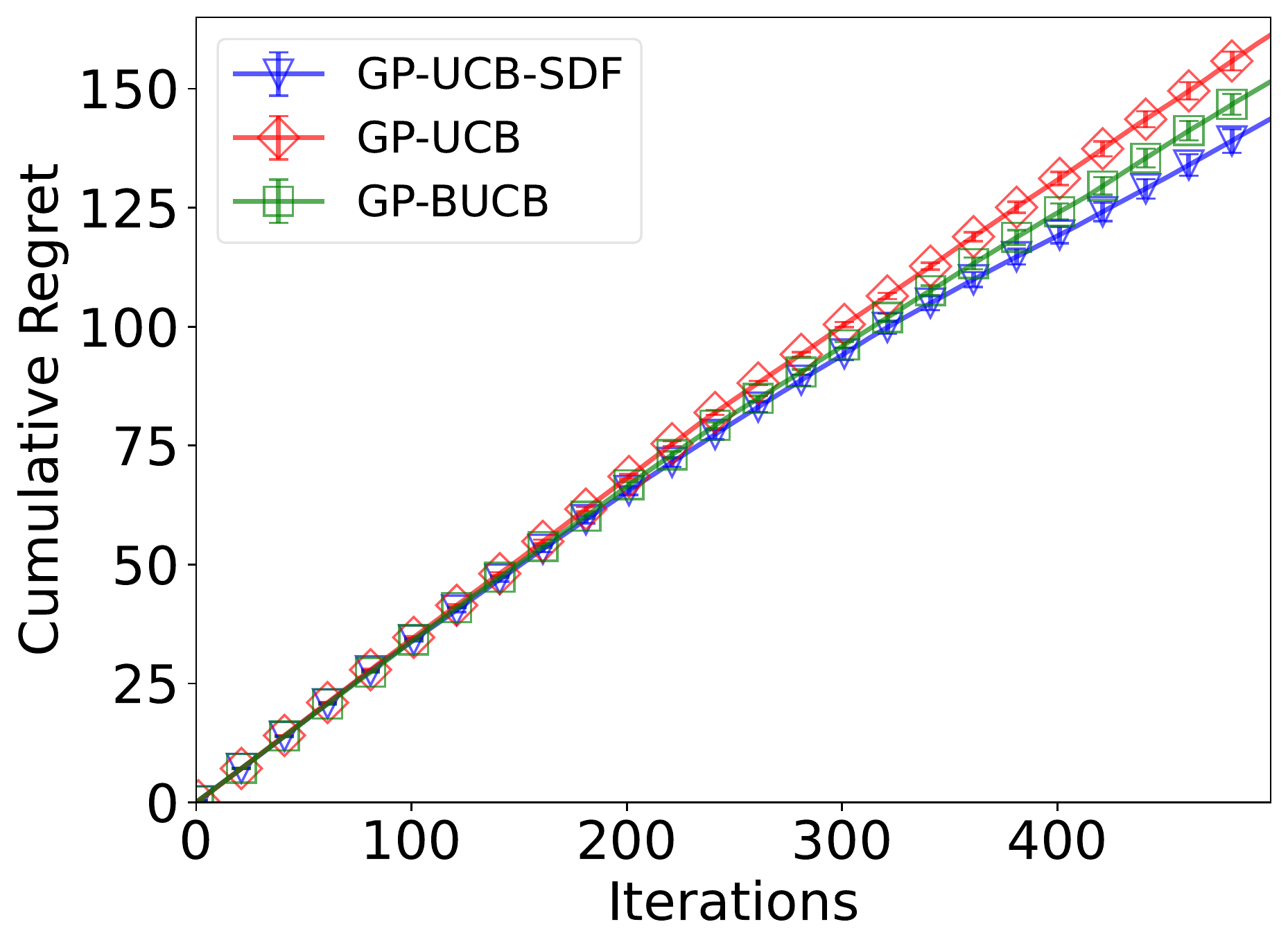}& \hspace{-5mm} 
         \includegraphics[width=0.25\linewidth]{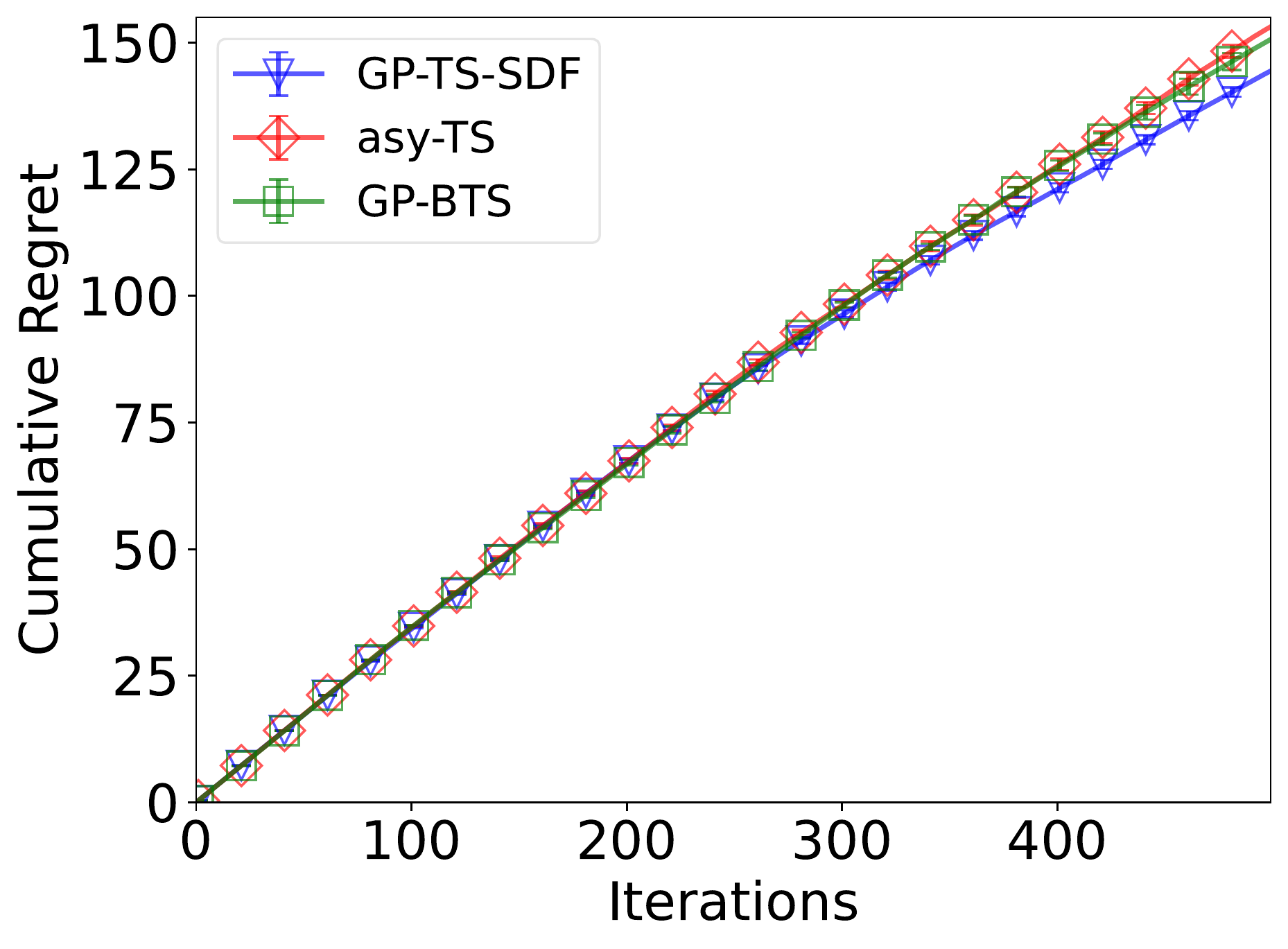}\\
         {(a)} & {(b)} & {(c)} & {(d)}
     \end{tabular}
     \caption{
     Cumulative regret of contextual GP bandit for (a) UCB-based and (b) TS-based methods for multi-task BO, and (c) UCB-based and (d) TS-based methods for non-stationary BO.
     }
     \label{fig:exp:contextual}
\end{figure*}

We now tune the hyperparameters of three different ML models to demonstrate the consistency of our algorithms. Specifically, we tune two hyperparameters of SVM (i.e., the penalty parameter and RBF kernel parameter) used for diabetes diagnosis, three hyperparameters of convolution neural networks (CNNs) (i.e., the batch size, learning rate, and learning rate decay) using the MNIST dataset, and the same three hyperparameters of logistic regression (LR) used for breast cancer detection. The experimental results on hyperparameter tuning for SVM and CNN are plotted in Fig.~\ref{fig:exp:automl}, including the results for UCB- and TS-based algorithms for both settings. The results for hyperparameter tuning of LR are shown in Fig.~\ref{fig:exp:automl_lr}. The figures show that our \ref{alg:GP-UCB-SDF} and \ref{alg:GP-TS-SDF} consistently outperform the corresponding baselines in both settings with stochastic and deterministic delay distributions. The results here verify the practical effectiveness of our proposed algorithms in real-world problems.

\subsection{Real-world Experiments on Contextual Gaussian Process Bandits}
In this section, we apply our algorithms to two real-world contextual GP bandit problems (Section~\ref{sec:contx_bandits}). Our first experiment performs multi-task BO in the contextual GP bandit setting following the work of~\citet{NIPS11_krause2011contextual}. We use the tabular benchmark dataset for SVM hyperparameter tuning from the work of~\citet{wistuba2015learning}, which consists of the validation accuracy evaluated on a discrete domain of six SVM hyperparameters for $50$ different classification tasks. Each of the $50$ tasks is associated with a set of meta-features which are used as the contexts $z_t$. We sequentially tackle the $50$ tasks and perform hyperparameter tuning for each task for $30$ iterations, i.e., the context $z_t$ for every task is repeated consecutively for $30$ times and the contexts for all tasks arrive sequentially. 

Our second experiment performs non-stationary BO in the contextual GP bandit setting. In some real-world BO problems, the objective function is non-stationary (i.e., changing over time). For example, when we tune the hyperparameters of an ML model for diabetes diagnosis (Section~\ref{subsec:exp:automl}), we may need to perform the hyperparameter tuning task periodically because the size of the dataset is progressively growing due to sequentially arriving patient data. We divide the entire dataset into $20$ progressively growing datasets and assign an index $z_t\in\{0,\ldots,19\}$ to each dataset such that a smaller dataset has a smaller index. Then, we use the indices as the contexts because two datasets whose indices are close share many common data points and are hence expected to lead to similar performances. We let the contexts arrive in the order of their indices to simulate sequentially arriving patients and again repeat each context for $30$ iterations.

For both experiments, we use a stochastic delay distribution (i.e., Poisson distribution with $\mu=3$) and set $m=6$. The cumulative regrets of different algorithms are shown in Fig.~\ref{fig:exp:contextual}, in which both \ref{alg:GP-UCB-SDF} and \ref{alg:GP-TS-SDF} incur smaller regrets than the other baselines.

%% file: latex/related_work.tex

We have developed algorithms for the BO problems with stochastic delayed feedback. In the following, we briefly review the papers relevant to our problem.

{\bf Batch BO.} 
Batch BO is the closest BO setting to our BO-SDF problem. Many algorithms have been proposed for different variants of batch BO problems in the literature \citep{desautels2014parallelizing,AIStats16_gonzalez2016batch,daxberger2017distributed,kandasamy2018parallelised,ICML19_gong2019quantile,chowdhury2019batch,NeurIPS20_balandat2020botorch}. However, the delay in observing feedback observations in all these batch BO problems is fixed. Some batch BO algorithms are based on the UCB index \citep{desautels2014parallelizing,daxberger2017distributed}, while some others are based on TS \citep{kandasamy2018parallelised,chowdhury2019batch,ICML19_gong2019quantile}. The batch BO methods like `local penalization' \citep{AIStats16_gonzalez2016batch} and `fantasize' \citep{NeurIPS20_balandat2020botorch} can not work with stochastically delayed feedback and have no theoretical guarantees. All variants of batch BO problems can be treated as special cases of the BO-SDF problem by fixing the delays appropriately. Therefore, our UCB- and TS-based algorithms can be used for any batch BO problem.

{\bf Contextual Gaussian Process (GP) Bandits.}
Several prior works consider the availability of additional information to the learner before making the decision. This line of work is popularly known as contextual bandits \citep{WWW10_li2010contextual}. It has many applications in advertising, web search, and e-commerce. In contextual bandits, the mean reward is a function of context and often parameterized (e.g., linear \citep{WWW10_li2010contextual,AISTATS11_chu2011contextual,ICML13_agrawal2013thompson}, GLM \citep{ICML17_li2017provably}, or non-linear \citep{UAI13_valko2013finite}). Contextual GP bandits model the posterior belief of a non-linear function using GPs \citep{NIPS11_krause2011contextual}. 
To the best of our knowledge, our work is the first to introduce stochastic delayed feedback into contextual GP bandits.

{\bf Stochastic Delayed Feedback.}
Due to several practical applications, some works have explored fixed or stochastic delayed feedback in the stochastic multi-armed bandits (MAB) \citep{Book_slivkins2019introduction,Book_lattimore2020bandit} which is a simpler sequential decision-making framework than BO. In the MAB literature, the problems ranging from known fixed delay \citep{ICML13_joulani2013online}, unknown delay \citep{Arxiv17_zhong2017asynchronous}, unknown stochastic delay \citep{UAI17_vernade2017stochastic,AISTATS18_grover2018best}, and unknown delays in the adversarial setting \citep{AIStats19_li2019bandit} are well-studied. The setting of the unknown stochastic delay is further extended to the linear bandits \citep{NeurIPS19_zhou2019learning,ICML20_vernade2020linear} which is closest to our BO-SDF problem. 
However, these works on linear bandits assume that the black-box function is linear and has binary feedback, which may not hold in real-life applications.

%% file: latex/conclusion.tex

This paper has studied the Bayesian optimization (BO) problem with stochastic delayed feedback (BO-SDF). We have used Gaussian processes as a surrogate model for the posterior belief of the unknown black-box function. To exploit the information of queries with delayed feedback (i.e., incomplete queries) for updating the posterior mean function, we have used censored feedback that assigns the minimum function value (or `0') to the delayed feedback. It discourages the learner from selecting the incomplete queries (or queries near them) when choosing the next query and hence leads to more efficient exploration.  

We have proposed UCB- and TS-based algorithms for the BO-SDF problem and given their sub-linear regret guarantees. Interestingly, when delays are fixed in a certain way, batch BO becomes a special case of the BO-SDF problem. Hence, we can adapt our algorithms for the batch BO problem. We have shown that our algorithms are theoretically and empirically better and do not need the special initialization scheme required for the existing batch BO algorithms. We also extended our algorithms to contextual Gaussian process bandits with stochastic delayed feedback, which is in itself a non-trivial contribution.

As the learning environment keeps changing in real life, an interesting future direction will be considering the BO-SDF problem with a non-stationary function. We also plan to generalize our  algorithms to nonmyopic BO~\cite{dmitrii20a,ling16}, 
high-dimensional BO~\cite{NghiaAAAI18},
private outsourced BO~\cite{dmitrii20b}, preferential BO~\cite{topk}, federated/collaborative BO~\cite{dai2020federated,dai2021differentially,sim2021collaborative}, meta-BO~\cite{metaBO}, and multi-fidelity BO~\cite{yehong17,ZhangUAI19} settings, handle risks~\cite{nguyen2021value,nguyen2021conditional,SebICML22} and information-theoretic acquisition functions~\cite{LSE,TMES}, incorporate early stopping~\cite{dai2019}, and/or recursive reasoning~\cite{dai2020}, and consider its application to neural architecture search~\cite{nasi,nes} and inverse reinforcement learning~\cite{bala20}.
For applications with a huge budget of function evaluations, we like to couple our algorithms with the use of distributed/decentralized~\cite{LowUAI12,Chen13,LowRSS13,LowTASE15,HoangICML16,NghiaAAAI19,low15,Ruofei18}, online/stochastic~\cite{NghiaICML16,MinhAAAI17,LowECML14a,LowAAAI14,HaibinAPP}, or deep sparse GP models~\cite{haibinneurips19,HaibinAPP2} to represent the belief of the unknown objective function efficiently.

%% file: latex/appendix.tex

\section{Proofs of Theoretical Results}
\label{app:sec:proof}

\subsection{Proof of \cref{thm:confBound}}
We define $\varphi(x)=k(x,\cdot)$ as the feature map associated with the kernel $k$, i.e., $\varphi$ maps any input $x\in\mathcal{Q}$ into a feature in the (potentially infinite-dimensional) feature space of the RKHS associated with $k$. The reproducing property tells us that $f(x)=\langle \varphi(x),f\rangle_{k}=\varphi(x)^{\top}f$. Next, we define $V_t(\lambda)\triangleq \sum^{t-1}_{s=1}\varphi(x_s)\varphi(x_s)^{\top} + \lambda I$, $B_t\triangleq \sum^{t-1}_{s=1}\varphi(x_s)\tilde{y}_{s,t}$, and $\hat{\theta}^{W}_t\triangleq V_t(\lambda)^{-1}B_t$. Of note, the GP posterior mean~\eqref{eq:gp_posterior} can be equivalently expressed as $\mu_{t-1}(x)=\varphi(x)^{\top}V_t(\lambda)^{-1}B_t$, and the GP posterior variance~\eqref{eq:gp_posterior} can be equivalently written as $\sigma^2_{t-1}(x,x)=\sigma^2_{t-1}(x)=\lambda \varphi(x)V_t(\lambda)^{-1}\varphi(x)=\lambda \norm{\varphi(x)}_{V_t(\lambda)^{-1}}^2$. In this section, we prove \cref{thm:confBound}, which is restated below.

\confBound*
\begin{proof}
    Firstly, we can decompose $B_t$ as follows:
    \eqs{
        B_t = \sum^{t-1}_{s=1}\varphi(x_s)\tilde{y}_{s,t} = \sum^{t-1}_{s=1}\varphi(x_s)y_s\mathbbm{1}\{d_s\leq m\} + \sum^{t-1}_{s=t-m}\varphi(x_s)y_s\left(\mathbbm{1}\{d_s\leq t-s\}-\mathbbm{1}\{d_s\leq m\}\right).
    }
    Next, we can plug in this decomposition of $B_t$ to derive:
    \eq{
        \begin{split}
            |\mu_{t-1}(x) &- \rho_m f(x)| = |\varphi(x)^{\top}V_t(\lambda)^{-1}B_t - \rho_m \varphi(x)^{\top}f|\\
            &\leq \left|\varphi(x)^{\top}V_t(\lambda)^{-1}\left(\sum^{t-1}_{s=1}\varphi(x_s)y_s\mathbbm{1}\{d_s\leq m\} \right) - \rho_m \varphi(x)^{\top}f \right| + \\
            & \left| \varphi(x)^{\top}V_t(\lambda)^{-1}\left(\sum^{t-1}_{s=t-m}\varphi(x_s)y_s\left(\mathbbm{1}\{d_s\leq t-s\}-\mathbbm{1}\{d_s\leq m\}\right)\right)\right|.
        \end{split}
        \label{eq:use:Bt:decomposition}
    }
    in which the equality makes use of the expression for the GP posterior mean: $\mu_{t-1}(x)=\varphi(x)^{\top}V_t(\lambda)^{-1}B_t$, and the reproducing property: $f(x)=\varphi(x)^{\top}f$.
    Next, we will separately provide upper bounds on the two terms in \cref{eq:use:Bt:decomposition}.

    Firstly, We upper-bound the second term in \cref{eq:use:Bt:decomposition} as follows:
    \eq{
        \begin{split}
            \Big| \varphi(x)^{\top}V_t(\lambda)^{-1}&\left(\sum^{t-1}_{s=t-m}\varphi(x_s)y_s\left(\mathbbm{1}\{d_s\leq t-s\}-\mathbbm{1}\{d_s\leq m\}\right)\right)\Big|\\
            &\leq \norm{\varphi(x)}_{V_t(\lambda)^{-1}} \norm{V_t(\lambda)^{-1}\left(\sum^{t-1}_{s=t-m}\varphi(x_s)y_s\left(\mathbbm{1}\{d_s\leq t-s\}-\mathbbm{1}\{d_s\leq m\}\right)\right)}_{{V_t(\lambda)}}\\
            &\leq \sigma_{t-1}(x) \frac{1}{\sqrt{\lambda}} \norm{\sum^{t-1}_{s=t-m}\varphi(x_s)y_s\left(\mathbbm{1}\{d_s\leq t-s\}-\mathbbm{1}\{d_s\leq m\}\right)}_{V_t(\lambda)^{-1}}\\
            &\leq \sigma_{t-1}(x) \sum^{t-1}_{s=t-m} \norm{\varphi(x_s)y_s\left(\mathbbm{1}\{d_s\leq t-s\}-\mathbbm{1}\{d_s\leq m\}\right)}_{V_t(\lambda)^{-1}}\\
            &\leq \sigma_{t-1}(x) \cB_y\sum^{t-1}_{s=t-m}\norm{\varphi(x_s)}_{V_t(\lambda)^{-1}}\\
            &=\sigma_{t-1}(x)\frac{\cB_y}{\sqrt{\lambda}}\sum^{t-1}_{s=t-m}\sigma_{t-1}(x_s)\\
            &\leq \cB_y\sigma_{t-1}(x)\sum^{t-1}_{s=t-m}\sigma_{t-1}(x_s).
        \end{split}
        \label{eq:upper:bound:second:term:of:decomposition}
    }
    The second inequality and the equality have made use of the expression for the GP posterior variance:
    $\sigma^2_{t-1}(x)=\lambda \norm{\varphi(x)}_{V_t(\lambda)^{-1}}^2$, the third and last inequalities has made use of the fact that $\frac{1}{\sqrt{\lambda}} < 1$ because $\lambda > 1$ (note that if $\lambda > 0$, then there will be an additional multiplier factor of $1/\lambda$), the fourth inequality follows since we have assumed that all observations are bounded: $|y_s|\leq \cB_y,\forall s\geq1$.

    Next, define $\mu'_{t-1}(x)$ as the standard GP posterior mean \emph{without censoring the observations} (i.e., replace every $\tilde{y}_{s,t}$ by $y_s$ in \cref{eq:gp_posterior}).
    Now the first term in \cref{eq:use:Bt:decomposition} can be upper-bounded as:
    \eq{
        \begin{split}
            \Big|\varphi(x)^{\top}V_t(\lambda)^{-1}&\left(\sum^{t-1}_{s=1}\varphi(x_s)y_s\mathbbm{1}\{d_s\leq m\} \right) - \rho_m \varphi(x)^{\top}f \Big|\\
            &= \Big|\varphi(x)^{\top}V_t(\lambda)^{-1}\left(\sum^{t-1}_{s=1}\varphi(x_s)y_s(\rho_m + \eta_s) \right) - \rho_m \varphi(x)^{\top}f \Big|\\
            &\leq \Big|\varphi(x)^{\top}V_t(\lambda)^{-1}\sum^{t-1}_{s=1}\varphi(x_s)y_s\eta_s\Big| + \rho_m \Big|\varphi(x)^{\top}V_t(\lambda)^{-1}\sum^{t-1}_{s=1}\varphi(x_s)y_s -  \varphi(x)^{\top}f \Big|\\
            &= \Big|\varphi(x)^{\top}V_t(\lambda)^{-1}\sum^{t-1}_{s=1}\varphi(x_s)y_s\eta_s\Big| + \rho_m \Big|\mu'_{t-1}(x)-f(x) \Big|\\
            &\leq \Big|\varphi(x)^{\top}V_t(\lambda)^{-1}\sum^{t-1}_{s=1}\varphi(x_s)y_s\eta_s\Big| + \left(\cB_f+R\sqrt{2(\gamma_{t-1}+1+\log(2/\delta))}\right)\sigma_{t-1}(x).
        \end{split}
        \label{eq:proof:confidence:bound:1}
    }
    In the first equality, we have defined $\eta_s\triangleq \mathbbm{1}\{d_s\leq m\} - \rho_m$. In the second equality, we have again used the expression for the GP posterior mean and the reproducing property for $f$. The last inequality follows from Theorem 2 of~\citet{chowdhury2017kernelized}, which holds with the probability of $\geq 1-\delta/2$. We have also used the fact that $\rho_m\leq1$ in the last inequality.

    To bound the first term in \cref{eq:proof:confidence:bound:1}, we firstly define $\eta'_s\triangleq y_s \eta_s$, which is by definition a $\cB_y$-sub-Gaussian noise. Also define $\boldsymbol{\eta}'_{1:t-1}=[\eta'_s]_{s=1,\ldots,t-1}$ which is a $(t-1)\times1-$dimensional vector and define $\Phi_{t-1}\triangleq[\varphi(x_1),\ldots,\varphi(x_{t-1})]^{\top}$ which is a $(t-1)\times \infty$-dimensional matrix.
    These definitions imply that $\mathbf{K}_{t-1}=\Phi_{t-1}\Phi_{t-1}^{\top}$ and $V_t(\lambda)=\Phi_{t-1}^{\top}\Phi_{t-1}+\lambda I$. Also define $\lambda=1+\eta$\footnote{The value of $\lambda$ is set as $(1+2/T)$ \cite{chowdhury2017kernelized}. However, one can use a Doubling Trick \citep{Arxiv18_besson2018doubling} for unknown $T$.} with $\eta>0$.
    Based on these definitions, we can upper-bound $\norm{\sum^{t-1}_{s=1}\varphi(x_s)y_s\eta_s}_{V_t(\lambda)^{-1}}$ by
    \eq{
        \begin{split}
            \norm{\sum^{t-1}_{s=1}\varphi(x_s)y_s\eta_s}_{V_t(\lambda)^{-1}}&=\norm{\sum^{t-1}_{s=1}\varphi(x_s)\eta'_s}_{V_t(\lambda)^{-1}}=\sqrt{\left(\sum^{t-1}_{s=1}\eta'_s\varphi(x_s)^{\top}\right) V_t(\lambda)^{-1} \left(\sum^{t-1}_{s=1}\eta'_s\varphi(x_s)\right)}\\
            &=\sqrt{\boldsymbol{\eta}'^{\top}_{1:t-1}\Phi_{t-1}(\Phi_{t-1}^{\top}\Phi_{t-1}+\lambda I)^{-1}\Phi_{t-1}^{\top}\boldsymbol{\eta}'_{1:t-1}}\\
            &=\sqrt{\boldsymbol{\eta}'^{\top}_{1:t-1}\Phi_{t-1}\Phi_{t-1}^{\top}(\Phi_{t-1}\Phi_{t-1}^{\top}+\lambda I)^{-1}\boldsymbol{\eta}'_{1:t-1}}\\
            &=\sqrt{\boldsymbol{\eta}'^{\top}_{1:t-1} \mathbf{K}_{t-1}(\mathbf{K}_{t-1}+(1+\eta) I)^{-1}\boldsymbol{\eta}'_{1:t-1}}\\
            &\leq \sqrt{\boldsymbol{\eta}'^{\top}_{1:t-1} (\mathbf{K}_{t-1}+\eta I)(\mathbf{K}_{t-1}+(1+\eta) I)^{-1}\boldsymbol{\eta}'_{1:t-1}}\\
            &= \sqrt{\boldsymbol{\eta}'^{\top}_{1:t-1} 
            ((\mathbf{K}_{t-1}+\eta I)^{-1}+I)^{-1}
            \boldsymbol{\eta}'_{1:t-1}}\\
            &= \norm{\boldsymbol{\eta}'_{1:t-1}}_{((\mathbf{K}_{t-1}+\eta I)^{-1}+I)^{-1}}.
        \end{split}
        \label{eq:proof:use:concentration:from:gp:ts:paper:1}
    }
    In the equality in the third line, we have made use of the following matrix equality: $(\Phi_{t-1}^{\top}\Phi_{t-1}+\lambda I)^{-1}\Phi_{t-1}^{\top} = \Phi_{t-1}^{\top}(\Phi_{t-1}\Phi_{t-1}^{\top}+\lambda I)^{-1}$, and in the second last equality, we have used the matrix equality of $(\mathbf{K}_{t-1}+\eta I)(\mathbf{K}_{t-1}+\eta I + I)^{-1} = ((\mathbf{K}_{t-1}+\eta I)^{-1}+I)^{-1}$.
    Next, making use of Theorem 1 of~\citet{chowdhury2017kernelized} and following similar steps as those in the proof there, we have that with probability of $\geq 1-\delta/2$,
    \eq{
        \begin{split}
            \norm{\boldsymbol{\eta}'_{1:t-1}}_{((\mathbf{K}_{t-1}+\eta I)^{-1}+I)^{-1}}
            &\leq \cB_y\sqrt{ 2\log\frac{\sqrt{\text{det}((1+\eta)I + \mathbf{K}_{t-1})}}{\delta/2}}\\
            &\leq \cB_y\sqrt{2(\gamma_{t-1}+1+\log(2/\delta))}.
        \end{split}
        \label{eq:proof:use:concentration:from:gp:ts:paper:2}
    }
    
    Next, combining \cref{eq:proof:use:concentration:from:gp:ts:paper:1} and \cref{eq:proof:use:concentration:from:gp:ts:paper:2} above allows us to upper-bound the first term in \cref{eq:proof:confidence:bound:1}:
    \eq{
        \begin{split}
            \Big|\varphi(x)^{\top}V_t(\lambda)^{-1}\sum^{t-1}_{s=1}\varphi(x_s)y_s\eta_s\Big| &\leq \norm{\varphi(x)}_{V_t(\lambda)^{-1}} \norm{V_t(\lambda)^{-1}\sum^{t-1}_{s=1}\varphi(x_s)y_s\eta_s}_{V_t(\lambda)}\\
            &\leq \sigma_{t-1}(x) \frac{1}{\sqrt{\lambda}} \norm{\sum^{t-1}_{s=1}\varphi(x_s)y_s\eta_s}_{V_t(\lambda)^{-1}}\\
            &\leq \sigma_{t-1}(x) \cB_y\sqrt{2(\gamma_{t-1}+1+\log(2/\delta))},
        \end{split}
        \label{eq:proof:use:concentration:from:gp:ts:paper:3}
    }
    in which the second inequality has made use of the expression for the GP posterior variance: $\sigma^2_{t-1}(x)=\lambda \norm{\varphi(x)}_{V_t(\lambda)^{-1}}^2$, and the last inequality follows from \cref{eq:proof:use:concentration:from:gp:ts:paper:1} and \cref{eq:proof:use:concentration:from:gp:ts:paper:2} (and hence holds with probability of $\geq1-\delta/2$), as well as the fact that $\lambda>1$. Now, we can plug in \cref{eq:proof:use:concentration:from:gp:ts:paper:3} as an upper bound on the first term in \cref{eq:proof:confidence:bound:1}, and then use the resulting \cref{eq:proof:confidence:bound:1} as an upper bound on the first term in \cref{eq:use:Bt:decomposition}.
    Combining this with the upper bound on the second term of \cref{eq:use:Bt:decomposition} given by \cref{eq:upper:bound:second:term:of:decomposition}, we have that
    \als{
        |\mu_{t-1}(x) - \rho_m f(x)| &\leq \cB_y\sigma_{t-1}(x)\sum^{t-1}_{s=t-m}\sigma_{t-1}(x_s) + \left(\cB_f + R\sqrt{2(\gamma_{t-1}+1+\log(2/\delta))}\right)\sigma_{t-1}(x) +\\
        &\quad \sigma_{t-1}(x) \cB_y\sqrt{2(\gamma_{t-1}+1+\log(2/\delta))}\\
        &= \sigma_{t-1}(x) \left( \cB_y\sum^{t-1}_{s=t-m}\sigma_{t-1}(x_s) + \cB_f+ (R + \cB_y)\sqrt{2(\gamma_{t-1}+1+\log(2/\delta))} \right)\\
        &= \nu_t\sigma_{t-1}(x),
    }
    in which we have plugged in the definition of $\nu_t$ in the last equality. This completes the proof.
\end{proof}

\subsection{Proof of \cref{thm:ucbRegret}}
\label{app:subsec:proof:ucb}

\ucbRegret*
\begin{proof}
    To begin with, we can upper-bound the instantaneous regret $r_t=f(x^*)-f(x_t)$ as
    \eq{
        \begin{split}
            r_t &= f(x^*) - f(x_t) = \frac{1}{\rho_m}\left[\rho_m f(x^*) - \rho_m f(x_t)\right]\\
            &\leq \frac{1}{\rho_m}\left[ \mu_{t-1}(x^*) + \nu_t\sigma_{t-1}(x^*) - \rho_m f(x_t) \right]\\
            &\leq \frac{1}{\rho_m}\left[ \mu_{t-1}(x_t) + \nu_t\sigma_{t-1}(x_t) - \rho_m f(x_t) \right] \\
            &\leq \frac{1}{\rho_m}2\nu_t\sigma_{t-1}(x_t),
        \end{split}
        \label{app:proof:ucb:inst:regret}
    }
    in which the first and last inequalities have made use of \cref{thm:confBound}, and the second inequality follows because $x_t$ is selected using the UCB policy: $x_t = \argmax_{x\in\mathcal{Q}} \left(\mu_{t-1} + \nu_t\sigma_{t-1}(x)\right)$~\eqref{eq:ucb_value}. Now the cumulative regret can be upper-bounded as
    \eq{
        \Regret_T = \sum^T_{t=1}r_t \leq \frac{2}{\rho_m}\sum^T_{t=1}\nu_t\sigma_{t-1}(x_t) =\frac{2}{\rho_m}\sum^T_{t=1}\sigma_{t-1}(x_t)\beta_t + \frac{2}{\rho_m}\sum^T_{t=1}\sigma_{t-1}(x_t)\left(\cB_y\sum^{t-1}_{s=t-m}\sigma_{t-1}(x_s)\right),
        \label{eq:decompose:RT}
    }
    where we have plugged in the expression of $\nu_t$ in the equality.

    Next, we use $I(\mathbf{y}_{1:T};\mathbf{f}_{1:T})$ to denote the information gain from the noisy observations in the first $T$ iterations about the objective function $f$.
    The first term in \cref{eq:decompose:RT} can be upper-bounded as:
    \eq{
        \begin{split}
            \frac{2}{\rho_m}\sum^T_{t=1}\sigma_{t-1}(x_t)\beta_t &\leq \frac{2}{\rho_m}\beta_T\sum^T_{t=1}\sigma_{t-1}(x_t) \\
            &\leq \frac{2}{\rho_m}\beta_T\sqrt{T}\sqrt{\sum^T_{t=1}\sigma^2_{t-1}(x_t)}\\
            &\leq \frac{2}{\rho_m}\beta_T\sqrt{T}\sqrt{\sum^T_{t=1}\frac{1}{\log(1+\lambda^{-1})}\log(1+\lambda^{-1}\sigma^2_{t-1}(x_t))}\\
            &\leq \frac{2}{\rho_m}\sqrt{\frac{1}{\log(1+\lambda^{-1})}}\beta_T\sqrt{T}\sqrt{2\frac{1}{2}\sum^T_{t=1}\log(1+\lambda^{-1}\sigma^2_{t-1}(x_t))}\\
            &\leq \frac{2}{\rho_m}\sqrt{\frac{1}{\log(1+\lambda^{-1})}}\beta_T\sqrt{T}\sqrt{2 I(\mathbf{y}_{1:T};\mathbf{f})}\\
            &\leq \frac{2}{\rho_m}\sqrt{\frac{2}{\log(1+\lambda^{-1})}}\beta_T\sqrt{T}\sqrt{\gamma_T} \\
            &\leq \frac{2}{\rho_m}C_1\beta_T\sqrt{T\gamma_T}.
        \end{split}
        \label{eq:proof:ucb:1}
    }
    The first inequality follows since $\beta_t$ is monotonically increasing in $t$, the inequality in the second line follows because $\lambda^{-1}a \leq \frac{\lambda^{-1}}{\log(1+\lambda^{-1})} \log(1+\lambda^{-1} a)$ for $a\in(0,1)$ (substitute $a=\sigma^2_{t-1}(x_t)$), and the inequality in the fourth line follows from plugging in the expression for the information gain: $I(\mathbf{y}_{1:T};\mathbf{f})=\frac{1}{2}\sum^T_{t=1}\log(1+\lambda^{-1}\sigma^2_{t-1}(x_t))$ (Lemma 5.3 of~\citet{ICML10_srinival2010gaussian}).

    Next, we can derive an upper bound on the second term in \cref{eq:decompose:RT}:
    \eq{
        \begin{split}
            \frac{2}{\rho_m}\sum^T_{t=1}\sigma_{t-1}(x_t)&\left(\cB_y\sum^{t-1}_{s=t-m}\sigma_{t-1}(x_s)\right)=\frac{2}{\rho_m}\cB_y\sum^T_{t=1}\sum^{t-1}_{s=t-m}\sigma_{t-1}(x_t)\sigma_{t-1}(x_s)\\
            &\leq \frac{1}{\rho_m}\cB_y\sum^T_{t=1}\sum^{t-1}_{s=t-m}\left(\sigma^2_{t-1}(x_t)+\sigma^2_{t-1}(x_s)\right)\\
            &\leq \frac{1}{\rho_m}\cB_y\sum^T_{t=1}\sum^{t-1}_{s=t-m}\left(\sigma^2_{t-1}(x_t)+\sigma^2_{s-1}(x_s)\right)\\
            &\leq \frac{2m}{\rho_m}\cB_y\sum^T_{t=1}\sigma^2_{t-1}(x_t) \\
            &\leq \frac{2m}{\rho_m}\cB_y C_1^2 \gamma_T.
        \end{split}
        \label{eq:proof:ucb:2}
    }
    The first inequality has made use of the inequality of $ab \leq (a^2+b^2)/2$, the second inequality follows since $\sigma^2_{t-1}(x_s)\leq \sigma^2_{s-1}(x_s)$ which is because $\sigma^2_{t-1}(x_s)$ is calculated by conditioning on more input locations than $\sigma^2_{s-1}(x_s)$, and the last inequality follows easily from some of the intermediate steps in the derivations of \cref{eq:proof:ucb:1}.
    
    Lastly, we can plug in \cref{eq:proof:ucb:1} and \cref{eq:proof:ucb:2} as upper bounds on the two terms in \cref{eq:decompose:RT}, to obtain:
    \als{
        \Regret_T &\leq \frac{2}{\rho_m}C_1\beta_T\sqrt{T\gamma_T} + \frac{2m}{\rho_m}\cB_y C_1^2 \gamma_T =\frac{2}{\rho_m}\left( C_1\beta_T\sqrt{T\gamma_T} + m\cB_y C_1^2 \gamma_T \right).
    }
    It completes the proof.
\end{proof}

\subsection{Proof of \cref{thm:tsRegret}}
\label{app:subsec:proof:ts}
First we define $\beta_t\triangleq \cB_f+ (R + \cB_y)\sqrt{2(\gamma_{t-1}+1+\log(4/\delta))}$, $\nu_t\triangleq \cB_y\sum^{t-1}_{s=t-m}\sigma_{t-1}(x_s) + \beta_t$ and $c_t \triangleq \nu_t (1+\sqrt{2\log(|\mathcal{Q}|t^2)})$. Note that we have replaced the $\delta$ in the definition of $\beta_t$ (\cref{thm:confBound}) by $\delta/2$. We use $\mathcal{F}_{t-1}$ to denote the filtration containing the history of selected inputs and observed outputs up to iteration $t-1$. To begin with, we define two events $E^f(t)$ and $E^{f_t}(t)$ through the following two lemmas.
\begin{lem}
    \label{lemma:uniform_bound}
    Let $\delta \in (0, 1)$. Define $E^f(t)$ as the event that $|\mu_{t-1}(x) - \rho_m f(x)| \leq \nu_t \sigma_{t-1}(x)$ for all $x\in \mathcal{Q}$.
    We have that $\mathbb{P}\left[E^f(t)\right] \geq 1 - \delta / 2$ for all $t\geq 1$.
\end{lem}

Note that \cref{lemma:uniform_bound} is the same as \cref{thm:confBound} except that we have replaced the error probability of $\delta$ by $\delta/2$.
\begin{lem}
    \label{lemma:uniform_bound_t}
    Define $E^{f_t}(t)$ as the event that $|f_t(x) - \mu_{t-1}(x)| \leq \nu_t \sqrt{2\log(|\mathcal{Q}|t^2)} \sigma_{t-1}(x)$.
    We have that $\mathbb{P}\left[E^{f_t}(t) | \mathcal{F}_{t-1}\right] \geq 1 - 1 / t^2$ for any possible filtration $\mathcal{F}_{t-1}$.
\end{lem}
\cref{lemma:uniform_bound_t} makes use of the concentration of a random variable sampled from a Gaussian distribution, and its proof follows from Lemma 5 of~\citet{chowdhury2017kernelized}. Next, we define a set of \emph{saturated points} in every iteration, which intuitively represents those inputs that lead to large regrets in an iteration.
\begin{defi}
    \label{def:saturated_set}
    In iteration $t$, define the set of saturated points as
    \eqs{
        S_t = \{ x \in \mathcal{Q} : \rho_m \Delta(x) > c_t \sigma_{t-1}(x) \},
    }
    where $\Delta(x) = f(x^*) - f(x)$ and $x^* \in \arg\max_{x\in \mathcal{Q}}f(x)$.
\end{defi}

Based on this definition, we will later prove that we can get a lower bound on the probability that the selected input in iteration $t$ is unsaturated (\cref{lemma:prob_unsaturated}). To do that, we first need the following auxiliary lemma.
\begin{lem}
\label{lemma:uniform_lower_bound}
    For any filtration $\mathcal{F}_{t-1}$, conditioned on the event $E^f(t)$, we have that $\forall x\in \mathcal{Q}$,
    \eq{
        \mathbb{P}\left(f_t(x) > \rho_m f(x) | \mathcal{F}_{t-1}\right) \geq p,
    }
    where $p = \frac{1}{4e\sqrt{\pi}}$.
\end{lem}
\begin{proof}
    Adding and subtracting $\frac{\mu_{t-1}(x)}{\nu_t\sigma_{t-1}(x)}$ both sides of $\mathbb{P}\left(f_t(x) > \rho_m f(x) | \mathcal{F}_{t-1}\right)$, we get
    \eq{
        \begin{split}
            \mathbb{P}\left(f_t(x) > \rho_m f(x) | \mathcal{F}_{t-1}\right) &= \mathbb{P}\left(\frac{f_t(x)-\mu_{t-1}(x)}{\nu_t\sigma_{t-1}(x)} > \frac{\rho_m f(x)-\mu_{t-1}(x)}{\nu_t\sigma_{t-1}(x)} | \mathcal{F}_{t-1}\right)\\
            &\geq \mathbb{P}\left(\frac{f_t(x)-\mu_{t-1}(x)}{\nu_t\sigma_{t-1}(x)} > \frac{|\rho_m f(x)-\mu_{t-1}(x)|}{\nu_t\sigma_{t-1}(x)} | \mathcal{F}_{t-1}\right)\\
            &\geq \mathbb{P}\left(\frac{f_t(x)-\mu_{t-1}(x)}{\nu_t\sigma_{t-1}(x)} > 1 | \mathcal{F}_{t-1}\right)\\
            &\geq \frac{1}{4e\sqrt{\pi}},
    \end{split}
    }
    in which the second inequality makes use of \cref{lemma:uniform_bound} (note that we have conditioned on the event $E^f(t)$ here), and the last inequality follows from the Gaussian anti-concentration inequality: $\bP(z>a) \geq \exp(-a^2) / (4\sqrt{\pi}a)$ where $z\sim\mathcal{N}(0,1)$.
\end{proof}

The next lemma proves a lower bound on the probability that the selected input $x_t$ is unsaturated.
\begin{lem}
    \label{lemma:prob_unsaturated}
    For any filtration $\mathcal{F}_{t-1}$, conditioned on the event $E^f(t)$, we have that,
    \eqs{
        \mathbb{P}\left(x_t \in \mathcal{Q}\setminus S_t | \mathcal{F}_{t-1} \right) \geq p - 1/t^2.
    }
\end{lem}
\begin{proof}
    To begin with, we have that
    \eq{
        \mathbb{P}\left(x_t \in \mathcal{Q}\setminus S_t | \mathcal{F}_{t-1} \right) \geq \mathbb{P}\left( f_t(x^*) > f_t(x),\forall x \in S_t | \mathcal{F}_{t-1} \right).
        \label{eq:lower_bound_prob_unsaturated}
    }
    This inequality can be justified because the event on the right hand side implies the event on the left hand side. Specifically, according to \cref{def:saturated_set}, $x^*$ is always unsaturated because $\Delta(x^*)=0<c_t\sigma_{t-1}(x^*)$. Therefore, because $x_t$ is selected by $x_t = {\arg\max}_{x\in\mathcal{Q}}f_t(x)$, we have that if $f_t(x^*) > f_t(x),\forall x \in S_t$, then the selected $x_t$ is guaranteed to be unsaturated. Now conditioning on both events $E^f(t)$ and $E^{f_t}(t)$, for all $x\in S_t$, we have that
    \eq{
        \begin{split}
            f_t(x) \leq \rho_m f(x) + c_t\sigma_{t-1}(x) \leq \rho_m f(x) + \rho_m \Delta(x)=\rho_mf(x) + \rho_mf(x^*) - \rho_mf(x) = \rho_mf(x^*),
        \end{split}
        \label{eq:bound_ftx_ftstar}
    }
    in which the first inequality follows from \cref{lemma:uniform_bound} and \cref{lemma:uniform_bound_t} and the second inequality makes use of \cref{def:saturated_set}.
    Next, separately considering the cases where the event $E^{f_t}(t)$ holds or not and making use of \cref{eq:lower_bound_prob_unsaturated} and \cref{eq:bound_ftx_ftstar}, we have that
    \eq{
        \begin{split}
            \mathbb{P}\left(x_t \in \mathcal{Q}\setminus S_t | \mathcal{F}_{t-1} \right) &\geq 
            \mathbb{P}\left( f_t(x^*) > f_t(x),\forall x \in S_t | \mathcal{F}_{t-1} \right)\\
            &\geq \mathbb{P}\left( f_t(x^*) > f(x^*) | \mathcal{F}_{t-1} \right) - \mathbb{P}\left(\overline{E^{f_t}(t)} | \mathcal{F}_{t-1}\right)\\
            &\geq p - 1 / t^2,
        \end{split}
        \label{eq:unsaturated_prob_plugin_2}
    }
    in which the last inequality has made use of  \cref{lemma:uniform_bound_t} and \cref{lemma:uniform_lower_bound}.
\end{proof}

Next, we use the following lemma to derive an upper bound on the expected instantaneous regret.
\begin{lem}
    \label{lemma:upper_bound_expected_regret}
    For any filtration $\mathcal{F}_{t-1}$, conditioned on the event $E^f(t)$, we have that,
    \eqs{
        \bE [r_t | \mathcal{F}_{t-1}] \leq \frac{11c_t}{\rho_m p} \bE \Lb \sigma_{t-1}(x_t) | \mathcal{F}_{t-1} \Rb + \frac{2\cB_f}{t^2}.
    }
\end{lem}
\begin{proof}
    To begin with, define $\overline{x}_t$ as the unsaturated input with the smallest GP posterior standard deviation:
    \eq{
        \overline{x}_t = {\arg\min}_{x\in\mathcal{Q}\setminus S_t}\sigma_{t-1}(x).
    }
    This definition gives us:
    \eq{
        \mathbb{E}\left[\sigma_{t-1}(x_t) | \mathcal{F}_{t-1}\right] \geq \mathbb{E}\left[\sigma_{t-1}(x_t) | \mathcal{F}_{t-1}, x_t \in \mathcal{Q} \setminus S_t\right]\mathbb{P}\left(x_t \in \mathcal{Q} \setminus S_t | \mathcal{F}_{t-1}\right) \geq \sigma_{t-1}(\overline{x}_t)(p-1/t^2),
        \label{eq:x_bar:sigma}
    }
    in which the second inequality makes use of \cref{lemma:prob_unsaturated}, as well as the definition of $\overline{x}_t$.
    
    Next, conditioned on both events $E^{f}(t)$ and $E^{f_t}(t)$, we can upper-bound the instantaneous regret as
    \als{
        r_t &= f(x^*) - f(x_t) = \frac{1}{\rho_m}\left[\rho_m f(x^*) - \rho_m f(\overline{x}_t) + \rho_m f(\overline{x}_t) - \rho_m f(x_t)\right]\\
        &\leq \frac{1}{\rho_m}\left[\rho_m \Delta(\overline{x}_t) + f_t(\overline{x}_t) + c_t \sigma_{t-1}(\overline{x}_t) - f_t(x_t) + c_t \sigma_{t-1}(x_t)\right]\\
        &\leq \frac{1}{\rho_m}\left[c_t\sigma_{t-1}(\overline{x}_t) + f_t(\overline{x}_t) + c_t \sigma_{t-1}(\overline{x}_t) - f_t(x_t) + c_t \sigma_{t-1}(x_t)\right]\\
        &\leq \frac{1}{\rho_m}\left[2c_t\sigma_{t-1}(\overline{x}_t) + c_t \sigma_{t-1}(x_t)\right],
    }
    in which the first inequality follows from \cref{lemma:uniform_bound} and \cref{lemma:uniform_bound_t} as well as the definition of $\Delta(\cdot)$ (\cref{def:saturated_set}), the second inequality follows because $\overline{x}_t$ is unsaturated, and the last inequality follows because $f_t(\overline{x}_t) - f_t(x_t) \leq 0$ since $x_t = {\arg\max}_{x\in\mathcal{Q}}f_t(x)$. Next, by separately considering the cases where the event $E^{f_t}(t)$ holds and otherwise, we are ready to upper-bound the expected instantaneous regret:
    \eq{
        \begin{split}
            \mathbb{E}[r_t | \mathcal{F}_{t-1}] &\leq \frac{1}{\rho_m}\mathbb{E}[ 2c_t\sigma_{t-1}(\overline{x}_t) + c_t \sigma_{t-1}(x_t) | \mathcal{F}_{t-1}] + \frac{2\cB_f}{t^2}\\
            &\leq \frac{1}{\rho_m} \bE \Lb 2c_t \frac{\sigma_{t-1}(x_t)}{p-1/t^2} + c_t \sigma_{t-1}(x_t) | \mathcal{F}_{t-1} \Rb + \frac{2\cB_f}{t^2}\\
            &\leq \frac{c_t}{\rho_m}\Lp \frac{2}{p-1/t^2}+1 \Rp \bE \Lb \sigma_{t-1}(x_t) | \mathcal{F}_{t-1} \Rb + \frac{2\cB_f}{t^2}\\
            &\leq \frac{11c_t}{\rho_m p} \bE \Lb \sigma_{t-1}(x_t) | \mathcal{F}_{t-1} \Rb + \frac{2\cB_f}{t^2},
        \end{split}
        \label{eq:proof:ts:expected:regret:last:step}
    }
    in which the second inequality results from \cref{eq:x_bar:sigma}, and the last inequality follows because $\frac{2}{p-1/t^2} \leq 10/p$ and $1 \leq 1/p$.
\end{proof}

Next, we define the following stochastic process $(Y_t:t=0,\ldots,T)$, which we prove is a super-martingale in the subsequent lemma by making use of \cref{lemma:upper_bound_expected_regret}.
\begin{defi}
    \label{def:stochastic:process}
    Define $Y_0=0$, and for all $t=1,\ldots,T$,
    \eqs{
        \overline{r}_t=r_t \mathbb{I}\{E^{f}(t)\},\quad
        X_t = \overline{r}_t - \frac{11c_t}{\rho_m p} \sigma_{t-1}(x_t) - \frac{2\cB_f}{t^2}, \text{ and } \quad 
        Y_t=\sum^t_{s=1}X_s.
    }
\end{defi}

\begin{lem}
    \label{lemma:sup:martingale}
    $(Y_t:t=0,\ldots,T)$ is a super-martingale with respect to the filtration $\mathcal{F}_t$.
\end{lem}
\begin{proof}
    As $X_t = Y_t - Y_{t-1}$, we have
    \als{
        \mathbb{E}\left[Y_t - Y_{t-1} | \mathcal{F}_{t-1}\right] &= \mathbb{E}\left[X_t | \mathcal{F}_{t-1}\right]\\
        &=\mathbb{E}\left[\overline{r}_t - \frac{11c_t}{\rho_m p} \sigma_{t-1}(x_t) - \frac{2\cB_f}{t^2} | \mathcal{F}_{t-1}\right]\\
        &=\mathbb{E}\left[\overline{r}_t | \mathcal{F}_{t-1}\right] - \left[\frac{11c_t}{\rho_m p}\mathbb{E}\Lb \sigma_{t-1}(x_t) | \mathcal{F}_{t-1} \Rb + \frac{2\cB_f}{t^2}
        \right]\\
        &\leq 0.
    }
    When the event $E^{f}(t)$ holds, the last inequality follows from \cref{lemma:upper_bound_expected_regret}; when $E^{f}(t)$ is false, $\overline{r}_t=0$ and hence the inequality trivially holds.
\end{proof}

Lastly, we are ready to prove the upper bound of \ref{alg:GP-TS-SDF} on the cumulative regret $\Regret_T$ by applying the Azuma-Hoeffding Inequality to the stochastic process defined above.

\tsRegret*
\begin{proof}
    To begin with, we derive an upper bound on $|Y_t - Y_{t-1}|$:
    \eq{
        \begin{split}
        |Y_t - Y_{t-1}| &= |X_t| \leq |\overline{r}_t| + \frac{11c_t}{\rho_m p} \sigma_{t-1}(x_t) + \frac{2\cB_f}{t^2}\\
        &\leq 2\cB_f+ \frac{11c_t}{\rho_m p} + 2\cB_f\\
        &\leq \frac{1}{\rho_mp} \Lp 4\cB_f+ 11c_t \Rp,
        \end{split}
    }
    where the second inequality follows because $\sigma_{t-1}(x_t)\leq 1$, and the last inequality follows since $\frac{1}{\rho_m p}\geq 1$.
    Now we are ready to apply the Azuma-Hoeffding Inequality to $(Y_t:t=0,\ldots,T)$ with an error probability of $\delta/2$:
    \eq{
        \begin{split}
            \sum^T_{t=1}\overline{r}_t &\leq \sum^T_{t=1}\frac{11c_t}{\rho_m p} \sigma_{t-1}(x_t) + \sum^T_{t=1}\frac{2\cB_f}{t^2} + \sqrt{2\log(2/\delta)\sum^T_{t=1} \Lp \frac{1}{\rho_mp} \Lp 4\cB_f+ 11c_t \Rp \Rp^2 }\\
            &\leq \frac{11}{\rho_m p}(1+\sqrt{2\log(|\mathcal{Q}|T^2)}) \sum^T_{t=1} \nu_t \sigma_{t-1}(x_t) + \frac{B\pi^2}{3} + \frac{4\cB_f+ 11c_T}{\rho_mp} \sqrt{2T\log(2/\delta)}\\
            &\leq \frac{11}{\rho_m p}(1+\sqrt{2\log(|\mathcal{Q}|T^2)}) \left( C_1\beta_T\sqrt{T\gamma_T} + m\cB_y C_1^2 \gamma_T \right) + \frac{4\cB_f+ 11c_T}{\rho_mp} \sqrt{2T\log(2/\delta)} + \frac{B\pi^2}{3},
        \end{split}
        \label{eq:use:azuma:hoedding}
    }
    in which the second inequality makes use of the fact that $c_t$ is monotonically increasing and that $\sum^T_{t=1}1/t^2\leq\pi^2/6$, and the last inequality follows from the proof of \cref{thm:ucbRegret} (\cref{app:subsec:proof:ucb}) which gives an upper bound on $\sum^T_{t=1} \nu_t \sigma_{t-1}(x_t)$.
    Note that \cref{eq:use:azuma:hoedding} holds with probability $\geq1-\delta/2$.
    Also note that $\overline{r}_t=r_t$ with probability of $\geq 1-\delta/2$ because the event $E^f(t)$ holds with probability of $\geq1-\delta/2$ (\cref{lemma:uniform_bound}), therefore, the upper bound from \cref{eq:use:azuma:hoedding} is an upper bound on $\Regret_T=\sum^T_{t=1}r_t$ with probability of $1-\delta$.
    
    Lastly, we can simplify the regret upper bound from \cref{eq:use:azuma:hoedding} into asymptotic notation.
    Note that $c_t = \nu_t (1+\sqrt{2\log(|\mathcal{Q}|t^2)})$, $\nu_t = \cB_y\sum^{t-1}_{s=t-m}\sigma_{t-1}(x_s) + \beta_t$ and $\beta_t = \cB_f+ (R + \cB_y)\sqrt{2(\gamma_{t-1}+1+\log(4/\delta))}$. 
    To simplify notations, we analyze the asymptotic dependence while ignoring all log factors, and ignore the dependence on constants including $B$, $\cB_y$ and $R$.
    This gives us $\beta_T=\tilde{O}(\sqrt{\gamma_T})$ and $c_T = \tilde{O}( m+\sqrt{\gamma_T})$. As a result, the regret upper bound can be simplified as
    \als{
        \Regret_T &= \tilde{O}\Lp \frac{1}{\rho_m} \Lp \gamma_T\sqrt{T} + m\gamma_T \Rp  + \frac{1}{\rho_m}(m+\sqrt{\gamma_T})\sqrt{T} \Rp = \tilde{O}\Lp \frac{1}{\rho_m}\Lp \sqrt{T\gamma_T} (\sqrt{\gamma_T} + 1) + m(\gamma_T + \sqrt{T}) \Rp \Rp.
    }
    It completes the proof.
\end{proof}

\subsection{Time based Censored Feedback}
\label{assec:censoredTime}
Let $m$ be the amount of time the learner waits for the delayed feedback, $\bar\tau_t$ be the time of initiating the selection process for the $t$-th query, and $\tau_t$ be the time of starting the $t$-th query.
The censored feedback of $s$-th query before selecting $t$-th query is denoted by $\tilde{y}_{s,t}$, where $\tilde{y}_{s,t} = y_s\mathbbm{1}\{d_s \leq \min(m, \bar\tau_t - \tau_s)\}$. Now, the GP posterior mean and covariance functions can be expressed using available information of function queries and their censored feedback as follows:
\eq{
	\label{eq:gp_posteriorTime}
	\begin{split}
			\mu_{t-1}(x) &= \mathbf{k}_{t-1}(x)^\top\mathbf{K}_{t,\lambda}^{-1}\tilde{\mathbf{y}}_{t-1},\\
			\sigma_{t-1}^2(x,x') &= k(x,x')-\mathbf{k}_{t-1}(x)^\top\mathbf{K}_{t,\lambda}^{-1}\mathbf{k}_{t-1}(x'),
		\end{split}
	}
where $\mathbf{K}_{t,\lambda} = \mathbf{K}_{t-1}+\lambda\mathbf{I}$, $\mathbf{k}_{t-1}(x)= (k(x, x_{t'}))^{\top}_{\tau_{t'} \in [0, \bar\tau_t]}$, $\tilde{\mathbf{y}}_{t-1}= (\tilde{y}_{s,t})^{\top}_{\tau_s\in [0, \bar\tau_t]}$, $\mathbf{K}_{t-1}= (k(x_{t'}, x_{t''}))_{\tau_{t'},\tau_{t''} \in [0, \bar\tau_t]}$ and  
$\lambda$ is a regularization parameter to ensures $\mathbf{K}_{t,\lambda}$ is an invertible matrix. Our current regret analysis will work for this setting as well after appropriately changing the used variables.

\section{Theoretical Analysis of Contextual Gaussian Process Bandit}
\label{app:sec:contextual:gp:bandit}
Our regret upper bounds in \cref{thm:ucbRegret} and \cref{thm:tsRegret} are also applicable in the contextual GP bandit setting after some slight modifications to their proofs.
Note that in the contextual GP bandit setting (\cref{sec:contx_bandits}), the only major difference with the BO setting is that now the objective function $g$ depends on both a context vector $z_t\in\cZ$ and an input vector $x_t\in\cQ$, in which every $z_t$ is generated by the environment and $x_t$ is selected by our GP-UCB-SDF or GP-TS-SDF algorithms. Therefore, the most important modifications to the proofs of \cref{thm:ucbRegret} and \cref{thm:tsRegret} involve replacing the original input space of $\cQ$ by the new input space of $\cZ \times \cQ$ and accounting for the modified definition of regret in \cref{sec:contx_bandits}.

To begin with, with the same definition of $\nu_t$, it is easy to verify that \cref{thm:confBound} can be extended to the contextual GP bandit setting:
\begin{restatable}[Confidence Ellipsoid for Contextual GP Bandit]{thm}{confBoundContextual}
	\label{thm:confBound:contextual}
	With probability at least $1-\delta$,
    \eqs{
        |\mu_{t-1}(z,x) - \rho_m g(z,x)| \leq \nu_t \sigma_{t-1}(z,x), \forall z\in\cZ,x\in\cQ.
    }
\end{restatable}
The proof of \cref{thm:confBound:contextual} is the same as that of \cref{thm:confBound} except that the input $x$ in the proof of \cref{thm:confBound} is replaced by $(z,x)$.

\subsection{\ref{alg:GP-UCB-SDF} for Contextual Gaussian Process Bandit}

For the \ref{alg:GP-UCB-SDF} algorithm in the contextual GP bandit setting, the instantaneous regret $r_t$ can be analyzed in a similar way to \cref{app:proof:ucb:inst:regret} in the proof of \cref{thm:ucbRegret} (\cref{app:subsec:proof:ucb}):
\als{
    r_t &= g(z_t, x_t^\star) - g(z_t, x_t) = \frac{1}{\rho_m}\left[\rho_m g(z_t, x_t^\star) - \rho_m g(z_t, x_t)\right]\\
    &\leq \frac{1}{\rho_m}\left[ \mu_{t-1}(z_t, x_t^\star) + \nu_t\sigma_{t-1}(z_t, x_t^\star) - \rho_m g(z_t, x_t) \right]\\
    &\leq \frac{1}{\rho_m}\left[ \mu_{t-1}(z_t, x_t) + \nu_t\sigma_{t-1}(z_t, x_t) - \rho_m g(z_t, x_t) \right]\\
    &\leq \frac{1}{\rho_m}2\nu_t\sigma_{t-1}(z_t, x_t),
}
in which the first and last inequalities both make use of \cref{thm:confBound:contextual}, and the second inequality follows from the way in which $x_t$ is selected: $x_t = {\arg\max}_{x\in\cQ} \mu_{t-1}(z_t, x) + \nu_t\sigma_{t-1}(z_t, x)$. Following this, the subsequent steps in the proof in \cref{app:subsec:proof:ucb} immediately follow, hence producing the same regret upper bound as \cref{thm:ucbRegret}.

\subsection{\ref{alg:GP-TS-SDF} for Contextual Gaussian Process Bandit}
To begin with, the events $E^f(t)$ (\cref{lemma:uniform_bound}) and $E^{f_t}(t)$ (\cref{lemma:uniform_bound_t}) can be similarly defined:
\begin{lem}
    \label{lemma:uniform_bound:contextual}
    Let $\delta \in (0, 1)$. Define $E^f(t)$ as the event that $|\mu_{t-1}(z,x) - \rho_m g(z,x)| \leq \nu_t \sigma_{t-1}(z,x)$ for all $z\in\cZ,x\in\cQ$. We have that $\mathbb{P}\left[E^f(t)\right] \geq 1 - \delta / 2$ for all $t\geq 1$.
\end{lem}

\begin{lem}
    \label{lemma:uniform_bound_t:contextual}
    Define $E^{f_t}(t)$ as the event that $|g_t(z,x) - \mu_{t-1}(z,x)| \leq \nu_t \sqrt{2\log(|\mathcal{Z}||\mathcal{Q}|t^2)} \sigma_{t-1}(z,x)$.
    We have that $\mathbb{P}\left[E^{f_t}(t) | \mathcal{F}_{t-1}\right] \geq 1 - 1 / t^2$ for any possible filtration $\mathcal{F}_{t-1}$.
\end{lem}
The validity of \cref{lemma:uniform_bound:contextual} follows immediately from \cref{thm:confBound:contextual} (after replacing the error probability of $\delta$ by $\delta/2$ in the definition of $\nu_t$, which we have omitted here for simplicity), and the proof of \cref{lemma:uniform_bound_t:contextual} is the same as that of \cref{lemma:uniform_bound_t}. Next, we modify the definition of saturated points from \cref{def:saturated_set}:
\begin{defi}
    \label{def:saturated_set:contextual}
    In iteration $t$, given a context vector $z_t$, define the set of saturated points as
    \eqs{
        S_t = \{ x\in \cQ : \rho_m \Delta(z_t,x) > c_t \sigma_{t-1}(z_t,x) \},
    }
    where $\Delta(z_t,x) = g(z_t, x_t^\star) - g(z_t, x_t)$ and $x_t^\star = \argmax\limits_{x \in \cQ} g(z_t, x)$.
\end{defi}

Note that according to this definition, $x_t^\star$ is always unsaturated.
The next lemma is a counterpart of \cref{lemma:uniform_lower_bound}, and the proofs are the same.
\begin{lem}
    \label{lemma:uniform_lower_bound:contextual}
    For any filtration $\mathcal{F}_{t-1}$, conditioned on the event $E^f(t)$, we have that $\forall z\in\cZ, x\in \cQ$,
    \eqs{
        \mathbb{P}\left(g_t(z,x) > \rho_m g(z,x) | \mathcal{F}_{t-1}\right) \geq p,
    }
    where $p=\frac{1}{4e\sqrt{\pi}}$.
\end{lem}

The next lemma, similar to \cref{lemma:prob_unsaturated}, shows that in contextual GP bandit problems, the probability that the selected $x_t$ is unsaturated can also be lower-bounded.
\begin{lem}
    \label{lemma:prob_unsaturated:contextual}
    For any filtration $\mathcal{F}_{t-1}$, given a context vector $z_t$, conditioned on the event $E^f(t)$, we have that,
    \eqs{
        \mathbb{P}\left( x_t \in \cQ \setminus S_t | \mathcal{F}_{t-1} \right) \geq p - 1/t^2.
    }
\end{lem}
\begin{proof}
    To begin with, we have that
    \eq{
        \mathbb{P}\left(x_t \in \cQ \setminus S_t | \mathcal{F}_{t-1} \right) \geq \mathbb{P}\left( g_t(z_t,x_t^*) > g_t(z_t,x),\forall x \in S_t | \mathcal{F}_{t-1} \right).
        \label{eq:lower_bound_prob_unsaturated:contextual}
    }
    The validity of this equation can be seen following the same analysis of \cref{eq:lower_bound_prob_unsaturated}, i.e., the event on the right hand side implies the event on the left hand side. Now conditioning on both events $E^f(t)$ and $E^{f_t}(t)$, for all $x\in S_t$, we have that
    \eq{
        \begin{split}
            g_t(z_t,x) &\leq \rho_m g(z_t,x) + c_t\sigma_{t-1}(z_t,x) \leq \rho_m g(z_t,x) + \rho_m \Delta(z_t,x)\\
            &=\rho_m g(z_t,x) + \rho_m g(z_t,x_t^*) - \rho_m g(z_t,x) = \rho_m g(z_t,x_t^*),
        \end{split}
        \label{eq:bound_ftx_ftstar:contextual}
    }
    in which the first inequality follows from \cref{lemma:uniform_bound:contextual} and \cref{lemma:uniform_bound_t:contextual} and the second inequality makes use of \cref{def:saturated_set:contextual}.
    
    Next, separately considering the cases where the event $E^{f_t}(t)$ holds or not and making use of \cref{eq:lower_bound_prob_unsaturated:contextual} and \cref{eq:bound_ftx_ftstar:contextual}, we have that
    \als{
            \mathbb{P}\left(x_t \in \mathcal{Q}\setminus S_t | \mathcal{F}_{t-1} \right) &\geq 
            \mathbb{P}\left( g_t(z_t,x_t^*) > g_t(z_t,x),\forall x \in S_t | \mathcal{F}_{t-1} \right)\\
            &\geq \mathbb{P}\left( g_t(z_t,x_t^*) > \rho_m g(z_t,x_t^\star) | \mathcal{F}_{t-1} \right) - \mathbb{P}\left(\overline{E^{f_t}(t)} | \mathcal{F}_{t-1}\right)\\
            &\geq p - 1 / t^2,
    }
    where the last inequality has made use of \cref{lemma:uniform_bound_t:contextual} and \cref{lemma:uniform_lower_bound:contextual}.
\end{proof}

\begin{lem}
    \label{lemma:upper_bound_expected_regret:contextual}
    For any filtration $\mathcal{F}_{t-1}$, given a context vector $z_t$, conditioned on the event $E^f(t)$, we have that,
    \eqs{
        \bE [r_t | \mathcal{F}_{t-1}] \leq \frac{11c_t}{\rho_m p} \bE \Lb \sigma_{t-1}(z_t,x_t) | \mathcal{F}_{t-1} \Rb + \frac{2\cB_f}{t^2}.
    }
\end{lem}
\begin{proof}
    To begin with, given a context vector $z_t$, define $\overline{x}_t$ as the unsaturated input with the smallest GP posterior standard deviation:
    \eq{
        \overline{x}_t = {\arg\min}_{x\in\mathcal{Q}\setminus S_t}\sigma_{t-1}(z_t,x).
    }
    This definition gives us:
    \eq{
        \begin{split}
            \mathbb{E}\left[\sigma_{t-1}(z_t,x_t) | \mathcal{F}_{t-1}\right] &\geq \mathbb{E}\left[\sigma_{t-1}(z_t,x_t) | \mathcal{F}_{t-1}, x_t \in \mathcal{Q} \setminus S_t\right]\mathbb{P}\left(x_t \in \mathcal{Q} \setminus S_t | \mathcal{F}_{t-1}\right)\\
            &\geq \sigma_{t-1}(z_t,\overline{x}_t)(p-1/t^2),
        \end{split}
        \label{eq:x_bar:sigma:contextual}
    }
    in which the second inequality makes use of \cref{lemma:prob_unsaturated}, as well as the definition of $\overline{x}_t$. Next, conditioned on both events $E^{f}(t)$ and $E^{f_t}(t)$, we can upper-bound the instantaneous regret as
    \als{
        r_t &= g(z_t, x_t^\star) - g(z_t, x_t) \\
        &= \frac{1}{\rho_m}\left[\rho_m g(z_t, x_t^\star) - \rho_m g(z_t,\overline{x}_t) + \rho_m g(z_t,\overline{x}_t) - \rho_m g(z_t,x_t)\right]\\
        &\leq \frac{1}{\rho_m}\left[\rho_m \Delta(z_t,\overline{x}_t) + g_t(z_t,\overline{x}_t) + c_t \sigma_{t-1}(z_t,\overline{x}_t) - g_t(z_t,x_t) + c_t \sigma_{t-1}(z_t,x_t)\right]\\
        &\leq \frac{1}{\rho_m}\left[c_t\sigma_{t-1}(z_t,\overline{x}_t) + g_t(z_t,\overline{x}_t) + c_t \sigma_{t-1}(z_t,\overline{x}_t) - g_t(z_t,x_t) + c_t \sigma_{t-1}(z_t,x_t)\right]\\
        &\leq \frac{1}{\rho_m}\left[2c_t\sigma_{t-1}(z_t,\overline{x}_t) + c_t \sigma_{t-1}(z_t,x_t)\right],
    }
    where the first inequality follows from \cref{lemma:uniform_bound:contextual} and \cref{lemma:uniform_bound_t:contextual} as well as the definition of $\Delta(\cdot)$ (\cref{def:saturated_set:contextual}), the second inequality follows because $\overline{x}_t$ is unsaturated, and the last inequality follows because $g_t(z_t,\overline{x}_t) - g_t(z_t,x_t) \leq 0$ since $x_t = {\arg\max}_{x\in\mathcal{Q}}g_t(z_t,x)$.
    Next, the proof in \cref{eq:proof:ts:expected:regret:last:step} can be immediately applied, hence producing the upper bound on the instantaneous regret in this lemma.
\end{proof}

Now the remaining steps in the proof of \cref{thm:tsRegret} in \cref{app:subsec:proof:ts} immediately follow. Specifically, we can first define a stochastic process $(Y_t:t=0,\ldots,T)$ in the same way as \cref{def:stochastic:process}, and then use \cref{lemma:sup:martingale} to show that $(Y_t:t=0,\ldots,T)$ is super-martingale. Lastly, we can apply the Azuma Hoeffding Inequality to $(Y_t:t=0,\ldots,T)$ in the same way as \cref{thm:tsRegret}, which completes the proof.

\section{More Experimental Details}
\label{app:sec:more:experimenta:detail}
In our experiments, since it has been repeatedly observed that the theoretical value of $\beta_t$ is overly conservative~\cite{ICML10_srinival2010gaussian}, we set $\beta_t$ to be a constant ($\beta_t=1$) for all methods. In all experiments and for all methods, we use the SE kernel for the GP and optimize the GP hyperparameters by maximizing the marginal likelihood after every 10 iterations. In all experiments where we report the simple regret (e.g., \cref{fig:exp:synth}, \cref{fig:exp:automl}, etc.), we calculate the simple regret in an iteration using only those function evaluations which have converted (i.e., we ignore all pending observations). 

\subsection{Synthetic Experiment}
In the synthetic experiment, the objective function $f$ is sampled from a GP using the SE kernel with a lengthscale of $0.02$ defined on a 1-dimensional domain. The domain is an equally spaced grid within the range $[0,1]$ with a size $1000$.
The sampled function is normalized into the range of $[0,1]$.
\begin{figure}
     \centering
     \begin{tabular}{cc}
         \includegraphics[width=0.45\linewidth]{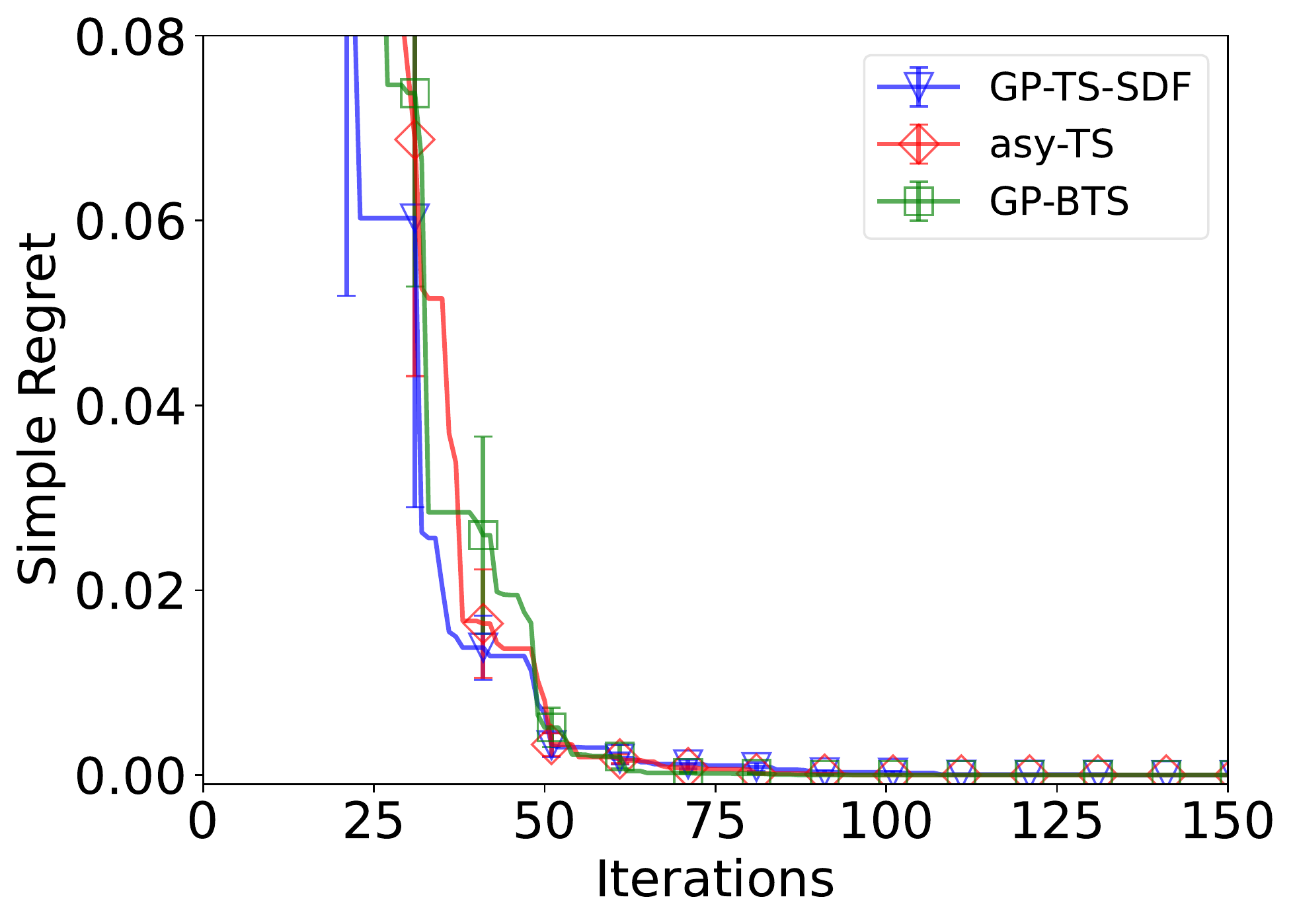}& 
         \includegraphics[width=0.45\linewidth]{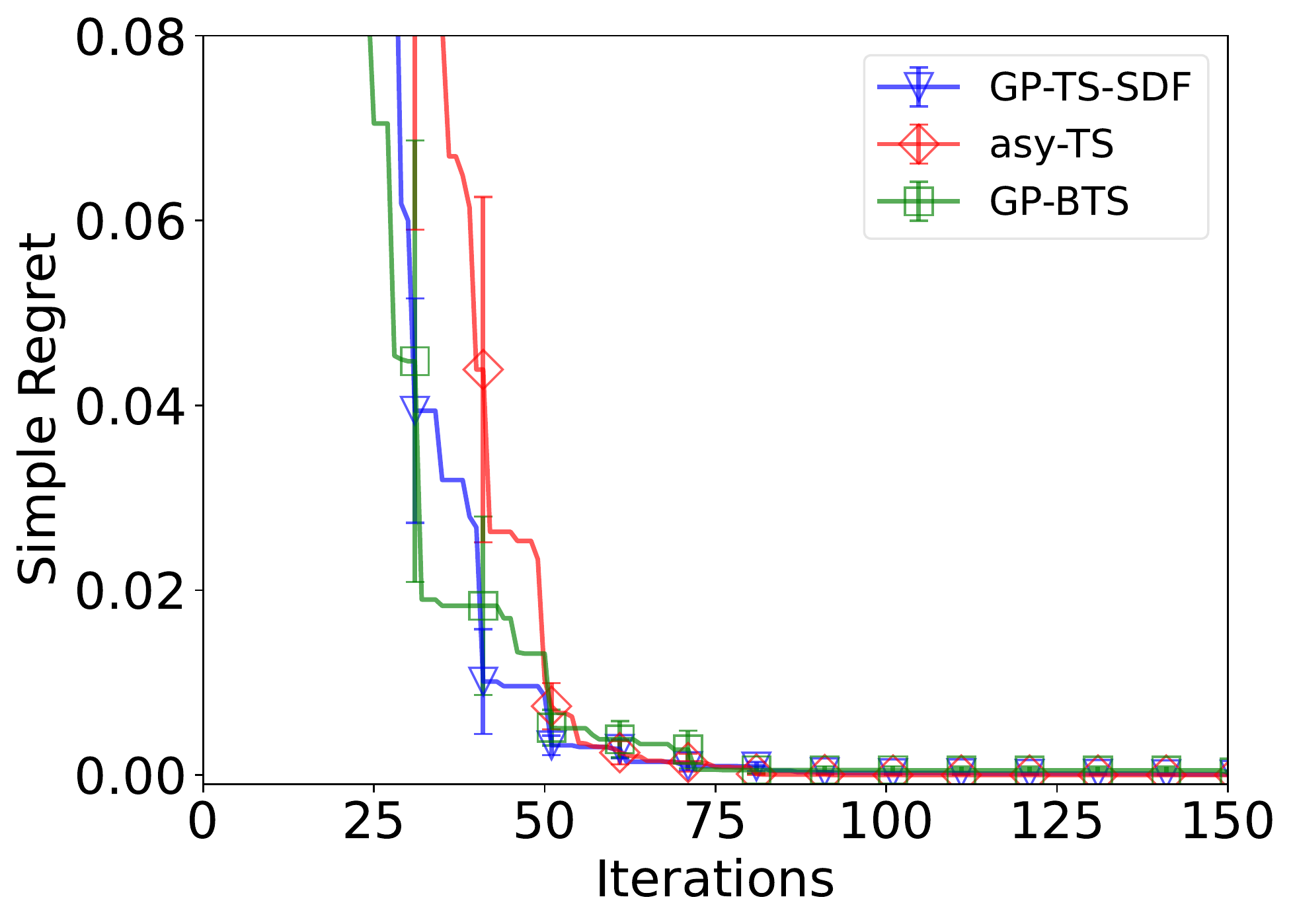}\\
         {(a)} & {(b)}
     \end{tabular}
\vspace{-2mm}
     \caption{
		Performances of GP-TS-SDF and other TS-based baseline methods for (a) stochastic and (b) deterministic delay distributions in the synthetic experiment (\cref{subsec:exp:synth}).
     }
     \label{fig:exp:synth:ts}
\vspace{-4mm}
\end{figure}

\subsection{Real-world Experiments}
\label{app:subsec:real:experiments}
In the SVM hyperparameter tuning experiment, the diabetes diagnosis dataset can be found at \url{https://www.kaggle.com/uciml/pima-indians-diabetes-database}. We use $70\%$ of the dataset as the training set and the remaining $30\%$ as the validation set. For every evaluated hyperparameter configuration, we train the SVM using the training set with the particular hyperparameter configuration and then evaluate the learned SVM on the validation set, whose validation accuracy is reported as the observations. We tune the penalty parameter within the range of $[10^{-4},100]$ and the RBF kernel parameter within $[10^{-4},10]$.

In the experiment on hyperparameter tuning for CNN, we use the MNIST dataset and follow the default training-testing sets partitions given by the PyTorch package. The CNN consists of one convolutional layer (with 8 channels and convolutional kernels of size 3), followed by a max-pooling layer (with a pooling size of 3), and then followed by a fully connected layer (with 8 nodes). We use the ReLU activation function and the Adam optimizer. We tune the batch size (within $[128,512]$), the learning rate (within $[10^{-6},1]$) and the learning rate decay (within $[10^{-6},1]$).

In the hyperparameter tuning experiment for LR, the breast cancer dataset we have adopted can be found at \url{https://archive.ics.uci.edu/ml/datasets/Breast+Cancer+Wisconsin+(Diagnostic)}. Similar to the SVM experiment, we use $70\%$ of the dataset as the training set and the remaining $30\%$ as the testing set. Here we tune three hyperparameters of the LR model: the batch size (within $[32, 128]$), the learning rate (within $[10^{-6},1]$) and learning rate decay (within $[10^{-6},1]$).

\subsection{Real-world Experiments on Contextual Gaussian Process Bandits}
For the multi-task BO experiment, the tabular benchmark on hyperparameter tuning of SVM is introduced by the work of~\cite{wistuba2015learning} and can be found at \url{https://github.com/wistuba/TST}. The dataset consists of 50 tasks, and each task corresponds to a different classification problem with a different dataset. The domain of hyperparameter configurations in this benchmark (which is shared by all 50 tasks) consists of discrete values for six hyperparameters: 3 binary parameters indicating (via one-hot encoding) whether the linear, polynomial, or RBF kernel is used, and the penalty parameter, the degree for the polynomial kernel, and the RBF kernel parameter. The size of the domain is 288. As a result, for each task, the validation accuracy for each one of the 288 hyperparameter configurations is recorded as the observation. As we have mentioned in the main text (\cref{subsec:exp:automl}), each of the 50 tasks is associated with some meta-features, and here we use the first six meta-features as the contexts. Each of the 50 tasks is associated with a 6-dimensional context vector $z_t$, which is used to characterize this particular task. For the non-stationary BO task, the dataset and the other experimental settings (e.g., the tuned hyperparameters, the range of the hyperparameters, etc.) are the same as the SVM hyperparameter tuning experiment in \cref{subsec:exp:automl}. Refer to \cref{app:subsec:real:experiments} for more details.

\vskip .15in
\hrule height1pt